\documentclass{tlp}

\newcommand\bcmdtab{\noindent\bgroup\tabcolsep=0pt%
  \begin{tabular}{@{}p{10pc}@{}p{20pc}@{}}}
\newcommand\ecmdtab{\end{tabular}\egroup}

  \title[Theory and Practice of Logic Programming]
        {Fuzzy Answer Set Computation via Satisfiability Modulo Theories}

  \author[M. Alviano and R. Pe\~naloza]
         {MARIO ALVIANO \\
         University of Calabria, Italy
         \and
         RAFAEL PE\~NALOZA \\
         Free University of Bozen-Bolzano, Italy}

\jdate{March 2003}
\pubyear{2003}
\pagerange{\pageref{firstpage}--\pageref{lastpage}}
\doi{XXX}

\usepackage{amsmath}
\usepackage{amssymb} 
\usepackage{comment}
\usepackage{xspace}
\usepackage{mathtools}
\usepackage{url}
\usepackage{thmtools}
\usepackage{thm-restate}
\usepackage[obeyDraft]{todonotes}
\usepackage{multirow}
\usepackage{enumitem}
\usepackage{booktabs}

\newtheorem{examplee}{Example}
\newenvironment{example}{\begin{examplee}}{\hfill $\blacksquare$\end{examplee}}

\newcommand{\naf}{\ensuremath{\raise.17ex\hbox{\ensuremath{\scriptstyle\mathtt{\sim}}}}\xspace} 

\def\T{\ensuremath{\mathcal{T}}\xspace}

\def\A{\ensuremath{\mathcal{A}}\xspace}
\def\R{\ensuremath{\mathbb{R}}\xspace}
\def\Q{\ensuremath{\mathbb{Q}}\xspace}

\def\heads{\ensuremath{\mathit{heads}}\xspace}

\def\At{\ensuremath{\mathit{At}}\xspace}

\def\SM{\ensuremath{\mathit{SM}}\xspace}
\def\G{\ensuremath{\mathcal{G}}\xspace}

\def\fuzzy{\ensuremath{\mathit{fuzzy}}\xspace}
\def\crisp{\ensuremath{\mathit{crisp}}\xspace}
\def\simp{\ensuremath{\mathit{simp}}\xspace}
\def\shift{\ensuremath{\mathit{shift}}\xspace}
\def\ite{\ensuremath{\mathit{ite}}\xspace}
\def\smt{\ensuremath{\mathit{smt}}\xspace}
\def\out{\ensuremath{\mathit{out}}\xspace}
\def\inn{\ensuremath{\mathit{inn}}\xspace}
\def\pos{\ensuremath{\mathit{pos}}\xspace}
\def\heads{\ensuremath{\mathit{heads}}\xspace}
\def\comp{\ensuremath{\mathit{comp}}\xspace}
\def\supp{\ensuremath{\mathit{supp}}\xspace}
\def\constraints{\ensuremath{\mathit{constraints}}\xspace}

\def\rank{\ensuremath{\mathit{rank}}\xspace}
\def\rcomp{\ensuremath{\mathit{rcomp}}\xspace}
\def\ocomp{\ensuremath{\mathit{ocomp}}\xspace}
\def\osupp{\ensuremath{\mathit{osupp}}\xspace}
\def\bool{\ensuremath{\mathit{bool}}\xspace}
\def\Godel{G\"odel\xspace}

\def\B{\ensuremath{\mathcal{B}}\xspace}
\newcommand{\Luka}{\L ukasiewicz\xspace}

\begin{document}

\label{firstpage}

\maketitle

\begin{abstract}
Fuzzy answer set programming (FASP) combines two declarative frameworks, answer set programming and fuzzy logic, in order to model reasoning by default over imprecise information.
Several connectives are available to combine different expressions; in particular the \Godel and \Luka fuzzy connectives
are usually considered, due to their properties.
Although the \Godel conjunction can be easily eliminated from rule heads, we show through complexity arguments 
that such a simplification is infeasible in general for all other connectives. 
The paper analyzes a translation of FASP programs into satisfiability modulo theories~(SMT), which in general produces quantified formulas because of the minimality of the semantics.
Structural properties of many FASP programs allow to eliminate the quantification, or to sensibly reduce the number of quantified variables.
Indeed, integrality constraints can replace recursive rules commonly used to force Boolean interpretations, and completion subformulas can guarantee minimality for acyclic programs with atomic heads.
Moreover, head cycle free rules can be replaced by shifted subprograms, whose structure depends on the eliminated head 
connective, so that ordered completion may replace the minimality check if also \Luka disjunction in rule bodies is acyclic.
The paper also presents and evaluates a prototype system implementing these translations.
\end{abstract}



\section{Introduction}\label{sec:intro}

Answer set programming (ASP) \cite{DBLP:journals/ngc/GelfondL91,DBLP:journals/amai/Niemela99,mare-trus-99} is a declarative language for knowledge representation, particularly suitable to model common non-monotonic tasks  such as reasoning by default, abductive reasoning, and belief revision \cite{bara-2002,DBLP:conf/nmr/MarekR04,DBLP:journals/ai/LinY02,DBLP:conf/kr/DelgrandeSTW08}.
If on the one hand ASP makes logic closer to the real world allowing for reasoning on incomplete knowledge, on the other hand it is unable to model imprecise information that may arise from the intrinsic limits of sensors, or the vagueness of natural language.
Fuzzy answer set programming (FASP) \cite{DBLP:journals/amai/NieuwenborghCV07} overcomes this limitation by interpreting propositions with a truth degree in the real interval $[0,1]$.
Intuitively, the higher the degree assigned to a proposition, the \emph{more true} it is, with $0$ and $1$ denoting \emph{totally false} and \emph{totally true}, respectively.
The notion of fuzzy answer set, or fuzzy stable model, was recently extended to arbitrary propositional formulas \cite{DBLP:conf/jelia/LeeW14}.
\citeANP{DBLP:conf/jelia/LeeW14} also propose an example on modeling dynamic \emph{trust} in social networks, which inspired the following simplified scenario that clarifies how truth degrees increase the knowledge representation capability of ASP.

\begin{example}\label{ex:social}
A user of a social network may trust or distrust another user, and these are vague concepts that can be naturally modeled by truth degrees.
These degrees may change over time.
For example, if at some point $A$ has a conflict with $B$, it is likely that her distrust on $B$ will increase and her trust on $B$ will decrease.
These are non-monotonic concepts that can be naturally handled in FASP.
\end{example}

In practice, however, ASP offers many efficient solvers such as \textsc{dlv} \cite{DBLP:conf/datalog/AlvianoFLPPT10}, \textsc{cmodels} \cite{DBLP:conf/lpnmr/LierlerM04}, \textsc{clasp} \cite{DBLP:journals/ai/GebserKS12}, and \textsc{wasp} \cite{DBLP:conf/lpnmr/AlvianoDFLR13}, which is not the case for FASP.
A preliminary FASP solver for programs with atomic heads and \Luka conjunction, called \textsc{fasp}, was presented at 
ICLP'13 by \cite{DBLP:journals/tplp/AlvianoP13}.
It implements approximation operators and a translation into bilevel programming \cite{DBLP:journals/ijar/BlondeelSVC14}.
A more general solver, called \textsc{ffasp} 
\cite{DBLP:conf/ecai/MushthofaSC14},
is based on a translation into ASP for computing stable models whose truth degrees are in the set 
$\Q_k := \{i/k \mid i \in [0..k]\}$, for a fixed $k$.
In general, exponentially many $k$ must be tested for checking the existene of a stable model, which is infeasible in practice. Therefore, \textsc{ffasp} tests by default a limited set of values. 
Neither \textsc{fasp} nor \textsc{ffasp} accept nesting of negation, which would allow to encode \emph{choice rules}, a convenient way for guessing truth degrees without using auxiliary atoms \cite{DBLP:conf/jelia/LeeW14}.
Indeed, choice rules allow to check satisfiability of fuzzy propositional formulas without adding new atomic propositions.
Our aim is to provide a more flexible FASP solver supporting useful patterns like choice rules.

Satisfiability modulo theories (SMT) \cite{DBLP:series/faia/BarrettSST09} extends propositional logic with external background theories---e.g. real arithmetic \cite{DBLP:journals/tocl/Ratschan06,DBLP:journals/jar/AkbarpourP10}---for which specialized methods provide efficient decision procedures.
SMT is thus a good candidate as a target framework for computing fuzzy answer sets efficiently.
This is non-trivial because the minimality condition that fuzzy stable models must satisfy makes the problem hard for the second level of the polynomial hierarchy; indeed, the translation provided in Section~\ref{sec:translation} produces quantified theories in general.
However, structural properties of the program that decrease the complexity to NP can be taken into account in order to obtain more tailored translations.
For example, disabling head connectives and recursive definitions yields a compact translation into fuzzy propositional logic known as \emph{completion} \cite{DBLP:journals/tplp/JanssenVSC12}, which in turn can be expressed in SMT (see Section~\ref{sec:completion}).
Since completion is unsound for programs with recursive definitions, the notion of \emph{ordered completion} has arisen in the ASP literature \cite{DBLP:journals/amai/Ben-EliyahuD94,DBLP:conf/ecai/Janhunen04,DBLP:journals/amai/Niemela08,DBLP:journals/ai/AsuncionLZZ12}.
In a nutshell, stable models of ASP programs with atomic heads can be recasted in terms of program reducts and fixpoint of the immediate consequence operator, where the computation of the fixpoint defines a ranking of the derived atoms.
Fuzzy stable models of programs with atomic heads can also be defined in terms of reducts and fixpoint of the immediate consequence operator \cite{DBLP:journals/tplp/JanssenVSC12}, although the notion of ranking can be extended to FASP only when recursive \Luka disjunction is disabled.
Using these notions, ordered completion is defined for FASP programs in Section~\ref{sec:ordered}.

In ASP, completion and ordered completion are also applicable to disjunctive programs having at most one recursive atom in each rule head.
Such programs, referred to as \emph{head cycle free} (HCF) \cite{DBLP:journals/amai/Ben-EliyahuD94}, are usually translated into equivalent programs with atomic heads by a so-called \emph{shift} \cite{DBLP:journals/tocl/EiterFW07}.
The same translation also works for HCF FASP programs using \Luka disjunction in rule heads.
On the other hand, \Luka conjunction and \Godel disjunction require more advanced constructions (Section~\ref{sec:shift}) which introduce recursive \Luka disjunction in rule bodies to restrict auxiliary atoms to
be Boolean.
Such rules are handled by integrality constraints in the theory produced by the completion, while they inhibit the application of the ordered completion.
As in ASP, the shift is unsound in general for FASP programs with head cycles, and complexity arguments given in Section~\ref{sec:hard} prove that it is unlikely that head connectives other than \Godel conjunction can be eliminated in general.

The general translation into SMT, completion, and ordered completion are implemented in a new FASP solver called \textsc{fasp2smt} (\url{http://alviano.net/software/fasp2smt/}; see Section~\ref{sec:experiment}).
\textsc{fasp2smt} uses \textsc{gringo} \cite{DBLP:conf/lpnmr/GebserKKS11} to obtain a ground representation of the input program, and \textsc{z3} \cite{DBLP:conf/tacas/MouraB08} to solve SMT instances encoding ground programs.
Efficiency of \textsc{fasp2smt} is compared with the previously implemented solver \textsc{ffasp} \cite{DBLP:conf/ecai/MushthofaSC14}, showing strengths and weaknesses of the proposed approach.

\section{Background}\label{sec:background}

We briefly recall the syntax and semantics of FASP \cite{DBLP:journals/amai/NieuwenborghCV07,DBLP:conf/jelia/LeeW14} and SMT \cite{DBLP:series/faia/BarrettSST09}.
Only the notions needed for the paper are introduced; for example, we only consider real arithmetic for SMT.

\subsection{Fuzzy Answer Set Programming}

Let $\B$ be a fixed set of propositional atoms.
A \emph{fuzzy atom} (\emph{atom} for short) is either a propositional atom from $\B$, or a numeric constant in $[0,1]$.
\emph{Fuzzy expressions} are defined inductively as follows:
every atom is a fuzzy expression;
if $\alpha$ is a fuzzy expression then $\naf \alpha$ is a fuzzy expression, where \naf denotes \emph{negation as failure};
if $\alpha$ and $\beta$ are fuzzy expressions, and $\odot \in \{\otimes,\oplus,\veebar,\barwedge\}$ is a connective, 
$\alpha \odot \beta$ is a fuzzy expression. 
Connectives $\otimes,\oplus$ are known as the \Luka connectives, and $\veebar,\barwedge$ are the \Godel
connectives.
A \emph{head expression} is a fuzzy expression of the form $p_1 \odot \cdots \odot p_n$, where $n \geq 1$, $p_1,\ldots,p_n$ 
are atoms, and $\odot \in \{\otimes,\oplus,\veebar,\barwedge\}$.
A \emph{rule} is of the form $\alpha \leftarrow \beta$, where $\alpha$ is a head expression, and $\beta$ is a fuzzy expression.
A \emph{FASP program} $\Pi$ is a finite set of rules.
Let $\At(\Pi)$ denote the set of atoms used by $\Pi$.

A \emph{fuzzy interpretation} $I$ for a FASP program $\Pi$ is a function $I : \B \rightarrow [0,1]$ mapping 
each propositional atom of $\B$ into a truth degree in $[0,1]$.
$I$ is extended to fuzzy expressions as follows:
$I(c) = c$ for $c \in [0,1]$;
$I(\naf \alpha) = 1 - I(\alpha)$;
$I(\alpha \otimes \beta) = \max\{I(\alpha) + I(\beta) - 1, 0\}$;
$I(\alpha \oplus \beta) = \min\{I(\alpha) + I(\beta), 1\}$;
$I(\alpha \veebar \beta) = \max\{I(\alpha), I(\beta)\}$; and
$I(\alpha \barwedge \beta) = \min\{I(\alpha), I(\beta)\}$.
%
%
$I$ satisfies a rule $\alpha \leftarrow \beta$ ($I \models \alpha \leftarrow \beta$) if 
$I(\alpha) \geq I(\beta)$; $I$ is a model of a FASP program $\Pi$, denoted $I \models \Pi$, if $I \models r$ for each $r \in \Pi$. 
$I$ is a \emph{stable model} of the FASP program $\Pi$ if $I \models \Pi$ and there is no interpretation $J$ such that 
$J \subset I$ and $J \models \Pi^I$, where the \emph{reduct} $\Pi^I$ is obtained from $\Pi$ by replacing each occurrence of a 
fuzzy expression $\naf \alpha$ by the constant $1-I(\alpha)$.
Let $\SM(\Pi)$ denote the set of stable models of $\Pi$.
A program $\Pi$ is \emph{coherent} if $\SM(\Pi) \neq \emptyset$; otherwise, $\Pi$ is \emph{incoherent}.
Two programs $\Pi,\Pi'$ are equivalent w.r.t.\ a crisp set $S \subseteq \B$, denoted $\Pi \equiv_S \Pi'$, if $|\SM(\Pi)| = |\SM(\Pi')|$ and $\{I \cap S \mid I \in \SM(\Pi)\} = \{I \cap S \mid I \in \SM(\Pi')\}$, where $I \cap S$ is the interpretation assigning $I(p)$ to all $p \in S$, and 0 to all $p \notin S$.

\begin{example}\label{ex:social:encoding}
Consider the scenario described in Example~\ref{ex:social}.
Let $U$ be a set of users, and $[0..T]$ the timepoints of interest, for some $T \geq 1$.
Let $\mathit{trust}(x,y,t)$ be a propositional atom expressing that $x \in U$ trusts $y \in U$ at time $t \in [0..T]$.
Similarly, $\mathit{distrust}(x,y,t)$ represents that $x$ distrusts $y$ at time $t$, and $\mathit{conflict}(x,y,t)$ encodes that $x$ has a conflict with $y$ at time $t$.
The social network example can be encoded by the FASP program $\Pi_1$ containing the following rules, for all $x \in U$, $y \in U$, and $t \in [0..T-1]$:
\[
\begin{array}{rcl}
    \mathit{distrust}(x,y,t+1) & \!\!\!\leftarrow\!\!\! & \mathit{distrust}(x,y,t) \oplus \mathit{conflict}(x,y,t) \\
    \mathit{trust}(x,y,t+1) & \!\!\!\leftarrow\!\!\! & \mathit{trust}(x,y,t) \otimes \naf(\mathit{distrust}(x,y,t+1) \otimes \naf\mathit{distrust}(x,y,t))
\end{array}
\]
The second rule above states that the trust degree of $x$ on $y$ decreases when her distrust degree on $y$ increases.
A stable model $I$ of $\Pi_1 \cup \{\mathit{trust}(\mathit{Alice},\mathit{Bob},0) \leftarrow 0.8$, $\mathit{conflict}(\mathit{Alice},\mathit{Bob},1) \leftarrow 0.2\}$ is such that $I(\mathit{distrust}(\mathit{Alice},\mathit{Bob},2)) = 0.2$, and $I(\mathit{trust}(\mathit{Alice},\mathit{Bob},2)) = 0.6$.
\end{example}

ASP programs are FASP programs such that all head connectives are $\veebar$, all body connectives are $\barwedge$, and all numeric constants are $0$ or $1$.
Moreover, an ASP program $\Pi$ implicitly contains \emph{crispifying} rules of the form $p \leftarrow p \oplus p$, for all $p \in \At(\Pi)$.
In ASP programs, $\veebar$ and $\barwedge$ are usually denoted $\vee$ and $\wedge$, respectively.

\subsection{Satisfiability Modulo Theories}

Let $\Sigma = \Sigma^V \cup \Sigma^C \cup \Sigma^F \cup \Sigma^P$ be a \emph{signature} where $\Sigma^V$ is a set of 
\emph{variables}, $\Sigma^C$ is a set of \emph{constant} symbols, $\Sigma^F$ is the set of binary \emph{function} symbols $\{+,-\}$, and $\Sigma^P$ is the set of binary \emph{predicate} symbols $\{<,\leq,\geq,>,=,\neq\}$.
\emph{Terms} and \emph{formulas} over $\Sigma$ are defined inductively, where we use infix notation for all binary
symbols.
Constants and variables are terms.
If $t_1,t_2$ are terms and $\odot \in \Sigma^F$ then $t_1 \odot t_2$ is a term.
If $t_1,t_2$ are terms and $\odot \in \Sigma^P$ then $t_1 \odot t_2$ is a formula.
If $\varphi$ is a formula and $t_1,t_2$ are terms then $\ite(\varphi,t_1,t_2)$ is a term (\ite stands for \emph{if-then-else}).
If $\varphi_1,\varphi_2$ are formulas and $\odot \in \{\vee, \wedge, \rightarrow, \leftrightarrow\}$ then $\varphi_1 \odot \varphi_2$ is a formula.
If $x$ is a variable and $\varphi$ is a formula then $\forall x.\varphi$ is a formula.
We consider only closed formulas, i.e., formulas in which all free variables are universally quantified.
For a term $t$ and integers $a,b$ with $a < b$, we use $t \in [a..b]$ in formulas to represent the subformula $\bigvee_{i = a}^b t = i$.
Similarly, for terms $t,t_1,t_2$, $t \in [t_1,t_2]$ represents $t_1 \leq t \wedge t \leq t_2$.
A $\Sigma$-theory $\Gamma$ is a set of $\Sigma$-formulas.

A $\Sigma$-structure \A is a pair $(\R,\cdot^\A)$, where $\cdot^\A$ is a mapping such that
$p^\A \in \R$ for each constant symbol $p$,
$(c)^\A = c$ for each number $c$,
$\odot^\A$ is the binary function $\odot$ over reals if $\odot \in \Sigma^F$, and the binary relation $\odot$ over reals if $\odot \in \Sigma^P$.
Composed terms and formulas are interpreted as follows:
for $\odot \in \Sigma^F$, $(t_1 \odot t_2)^\A = t_1^\A \odot t_2^\A$;
$\ite(\varphi,t_1,t_2)^\A$ equals $t_1^\A$ if $\varphi^\A$ is true, and $t_2^\A$ otherwise;
for $\odot \in \Sigma^P$, $(t_1 \odot t_2)^\A$ is true if and only if $t_1^\A \odot t_2^\A$;
for $\odot \in \{\vee,\wedge,\rightarrow,\leftrightarrow\}$, $(\varphi_1 \odot \varphi_2)^\A$ equals $\varphi_1^\A \odot \varphi_2^\A$ (in propositional logic);
$(\forall x.\varphi)^\A$ is true if and only if $\varphi[x/n]$ is true for all $n \in \R$, where $\varphi[x/n]$ is the formula obtained by substituting $x$ with $n$ in $\varphi$.
\A is a $\Sigma$-model of a theory $\Gamma$, denoted $\A \models \Gamma$, if $\varphi^\A$ is true for all $\varphi \in \Gamma$.

\begin{example}
Let $\Sigma^C$ be $\{p,q,s,z\}$, $x$ be a variable, and
$\Gamma_1=\{z \in [0,1], \forall x.(x \geq z)\}$ be a $\Sigma$-theory.
Any $\Sigma$-model of $\Gamma_1$ maps $z$ to 0.
If $\ite(p + q \leq 1, p + q, 1) \geq s \leftrightarrow (p \geq \ite(s - q \geq 0, s - q, 0) \wedge q \geq \ite(s - p \geq 0, s - p, 0))$ is added to $\Gamma_1$,
then any $\Sigma$-model of $\Gamma_1$ maps $z$ to 0, and $p,q,s$ to real numbers in the interval $[0,1]$.
\end{example}

\section{Structure Simplification}\label{sec:simplification}

The structure of FASP programs can be simplified through rewritings that leave at most one connective in each rule body \cite{DBLP:conf/ecai/MushthofaSC14}.
Essentially, a rule of the form $\alpha \leftarrow \beta \odot \gamma$, with $\odot \in \{\otimes,\oplus,\veebar,\barwedge\}$, is replaced by the rules
$\alpha \leftarrow p \odot q$, $p \leftarrow \beta$, and $q \leftarrow \gamma$, with $p$ and $q$ fresh atoms.
A further simplification, implicit in the translation into crisp ASP by \cite{DBLP:conf/ecai/MushthofaSC14}, eliminates $\barwedge$ in rule heads and $\veebar$ in rule bodies:
a rule of the form $p_1 \barwedge \cdots \barwedge p_n \leftarrow \beta$, $n \geq 2$, is equivalently replaced by 
$n$ rules $p_i \leftarrow \beta$, for $i \in [1..n]$;
and a rule of the form $\alpha \leftarrow \beta \veebar \gamma$ is replaced by $\alpha \leftarrow \beta$, $\alpha \leftarrow \gamma$.
Moreover, a rule of the form $\alpha \leftarrow \naf \beta$ can be equivalently replaced by the rules $\alpha \leftarrow \naf p$ and $p \leftarrow \beta$, where $p$ is a fresh atom.
Let $\simp(\Pi)$ be the program obtained from $\Pi$ by applying these substitutions. 
\begin{restatable}{proposition}{PropSimp}\label{prop:simp}
For every FASP program $\Pi$, 
it holds that $\Pi \equiv_{\At(\Pi)} \simp(\Pi)$,
i.e., $|\SM(\Pi)| = |\SM(\simp(\Pi))|$ and $\{I \cap \At(\Pi) \mid I \in \SM(\Pi)\} = \{I \cap \At(\Pi) \mid I \in \SM(\simp(\Pi))\}$.
\end{restatable}

\citeANP{DBLP:conf/ecai/MushthofaSC14} also simplify rule heads:
$\alpha \odot \beta \leftarrow \gamma$ is replaced by 
$p \odot q \leftarrow \gamma$, $p \leftarrow \alpha$, $\alpha \leftarrow p$, $q \leftarrow \beta$, and $\beta \leftarrow q$, where $p$ and $q$ are fresh atoms.
We do not apply these rewritings as they may inhibit other simplifications introduced in Section~\ref{sec:shift}.

\subsection{Hardness results}\label{sec:hard}

A relevant question is whether more rule connectives can be eliminated in order to further simplify  the structure of FASP programs.
We show that this is not possible, unless the polynomial hierarchy collapses, by adapting the usual reduction of 
2-QBF$_\exists$ satisfiability to ASP coherence testing \cite{DBLP:journals/amai/EiterG95}:
for $n > m \geq 1$, $k \geq 1$ and formula 
$\phi := \exists x_1,\ldots,x_m \forall x_{m+1},\ldots,x_n\ \bigvee_{i=1}^k L_{i,1} \wedge L_{i,2} \wedge L_{i,3}$, 
test the coherence of $\Pi_\phi$ below
\begin{eqnarray}
    x_i^T \vee x_i^F \leftarrow 1 && \forall i \in [1..n] \label{eq:hard:1} \\
    x_i^T \leftarrow \mathit{sat} \quad x_i^F \leftarrow \mathit{sat} \quad 0 \leftarrow \naf \mathit{sat} && \forall i \in [m+1..n] \label{eq:hard:3} \\
    \mathit{sat} \leftarrow \sigma(L_{i,1}) \wedge \sigma(L_{i,2}) \wedge \sigma(L_{i,3}) && \forall i \in [1..k] \label{eq:hard:4}
\end{eqnarray}
where $\sigma(x_i) := x_i^T$, and $\sigma(\neg x_i) := x_i^F$, for all $i \in [1..n]$.
$\Sigma^P_2$-hardness for FASP programs with $\veebar$ in rule heads is proved by defining a FASP program $\Pi_\phi^\veebar$ comprising (\ref{eq:hard:1})--(\ref{eq:hard:4}) (recall that $\vee$ is $\veebar$, and $\wedge$ is $\barwedge$).
This also holds if we replace $\wedge$ with $\otimes$ in (\ref{eq:hard:4}).
Another possibility is to replace $\vee$ with $\oplus$ in (\ref{eq:hard:1}), and add $p \leftarrow p \oplus p$ for all atoms in $\At(\Pi_\phi)$, showing $\Sigma^P_2$-hardness for FASP programs with $\oplus$ in rule heads, a result already proved by \citeN{DBLP:journals/ijar/BlondeelSVC14} with a different construction.

The same result also applies to $\otimes$, but we need a more involved argument.
Let $\Pi_\phi^\otimes$ be the program obtained from $\Pi_\phi$ by replacing $\wedge$ with $\otimes$, substituting the rule (\ref{eq:hard:1}) with the following three rules for each $i \in [1..n]$:
\begin{eqnarray*}
    x_i^T \otimes x_i^F \leftarrow 0.5 \quad &
    x_i^T \otimes x_i^T \otimes x_i^T \leftarrow x_i^T \otimes x_i^T \quad &
    x_i^F \otimes x_i^F \otimes x_i^F \leftarrow x_i^F \otimes x_i^F
\end{eqnarray*}
For all interpretations $I$, the first rule enforces $I(x_i^T) + I(x_i^F) \geq 1.5$.
The second rule enforces $3 \cdot I(x_i^T) - 2 \geq 2 \cdot I(x_i^T) - 1$ whenever $2 \cdot I(x_i^T) - 1 > 0$, i.e., $I(x_i^T) \geq 1$ whenever $I(x_i^T) > 0.5$.
Similarly, the third rule enforces $I(x_i^F) \geq 1$ whenever $I(x_i^F) > 0.5$.
Hence, one of $x_i^T,x_i^F$ is assigned 1, and the other 0.5.
Since conjunctions are modeled by $\otimes$, and each conjunction contains three literals whose interpretation is either 0.5 or 1, it follows that the interpretation of the conjunction is 1 if all literals are 1, and at most 0.5 otherwise.
Hence, $\phi$ is satisfiable if and only if $\Pi_\phi^\otimes$ is coherent.

%

\begin{restatable}{theorem}{ThmHard}\label{thm:hard}
Checking coherence of FASP programs is $\Sigma^P_2$-hard already in the following cases:
(i) all connectives are $\otimes$;
(ii) head connectives are $\veebar$, and body connectives are $\barwedge$ (or $\otimes$); and
(iii) head connectives are $\oplus$, and body connectives are $\barwedge$ (or $\otimes$) and $\oplus$.
\end{restatable}

\subsection{Shifting heads}\label{sec:shift}

Theorem~\ref{thm:hard} shows that $\oplus$, $\otimes$, and $\veebar$ cannot be eliminated from rule heads in general by a polytime translation, unless the polynomial hierarchy collapses.
This situation is similar to the case of disjunctions in ASP programs, which cannot be eliminated 
either. 
However, \emph{head cycle free} (HCF) programs admit
a translation known as \emph{shift} that eliminates $\vee$ preserving stable models \cite{DBLP:journals/tocl/EiterFW07}.
We extend this idea to FASP connectives.
The definition of HCF programs relies on the notion of \emph{dependency graph}.
Let $\pos(\alpha)$ denote the set of propositional atoms occurring in $\alpha$ but not under the scope of any $\naf$ symbol.
The dependency graph $\G_\Pi$ of a FASP program $\Pi$ has vertices $\At(\Pi)$, and an arc $(p,q)$ if there is a rule $\alpha \leftarrow \beta \in \Pi$ such that $p \in \pos(\alpha)$, and $q \in \pos(\beta)$.
A \emph{(strongly connected) component} of $\Pi$ is a maximal set containing pairwise reachable vertices of $\G_\Pi$.
A program $\Pi$ is \emph{acyclic} if $\G_\Pi$ is acyclic;
$\Pi$ is HCF if there is no rule $\alpha \leftarrow \beta$ where $\alpha$ contains two atoms from the same component of $\Pi$;
$\Pi$ has non-recursive $\odot \in \{\otimes,\oplus,\veebar,\barwedge\}$ in rule bodies if whenever $\odot$ occurs in the body of a rule $r$ of $\simp(\Pi)$ but not under the scope of a \naf symbol then for all $p \in H(r)$ and for all $q \in \pos(B(r))$ atoms $p$ and $q$ belong to different components of $\simp(\Pi)$.

\begin{example}
The program $\{p \leftarrow q \oplus \naf\naf p\}$ is acyclic.
Note that $\naf\naf p$ does not provide an arc to the dependency graph.
Adding the rule $q \otimes s \leftarrow p$ makes the program cyclic but still HCF because $q$ and $s$ belong
to two different components. 
If also $q \leftarrow s$ is added, then the program is no more HCF.
Finally, note that $\Pi_1$ in Example~\ref{ex:social:encoding} is acyclic.
\end{example}

It should now be clear why we decided not to reduce the number of head connectives in the
translation \simp defined at the beginning of this section. By removing a connective in the head of a rule
of an HCF program, we might produce a program that is not HCF.
Consider for example the HCF program $\{p \otimes q \otimes s \leftarrow 1\}$.
To reduce one of the occurrences of $\otimes$, we can introduce a fresh atom $\mathit{aux}$ that
stands for $q\otimes s$. 
However, $q$ and $s$ would belong to the same component of the resulting program
$\{p \otimes \mathit{aux} \leftarrow 1,$ $q \otimes s \leftarrow \mathit{aux},$ $\mathit{aux} \leftarrow q \otimes s\}$.

We now define the \emph{shift} of a rule for all types of head connectives.
The essential idea is to move all head atoms but one to the body (hence the name
shift).
To preserve stable models, this has to be repeated for all head atoms, and some additional conditions might be required.
For a rule of the form $p_1 \oplus \cdots \oplus p_n \leftarrow \beta$, the shift essentially mimics the original 
notion for ASP programs, and produces
\begin{equation}\label{eq:shift:oplus}
    p_i \leftarrow \beta \otimes \naf p_1 \otimes \cdots \otimes \naf p_{i-1} \otimes \naf p_{i+1} \otimes \cdots \otimes \naf p_n
\end{equation}
for all $i \in [1..n]$.
Intuitively, the original rule requires any model $I$ to satisfy the condition
$I(p_1) + \cdots + I(p_n) \geq I(\beta)$. This is the case if and only if 
$$I(p_i) \geq I(\beta) + \sum_{j \in [1..n], j \neq i} (1 - I(p_j)) - (n-1) = I(\beta) - \sum_{j \in [1..n], j \neq i} I(p_j);$$ 
i.e., if and only if \eqref{eq:shift:oplus} is satisfied, for all $i \in [1..n]$.
The shift of rules with other connectives in the head is more elaborate.
For $p_1 \otimes \cdots \otimes p_n \leftarrow \beta$, it produces
\begin{eqnarray}\label{eq:shift:otimes}
    p_i \leftarrow q \otimes (\beta \oplus \naf p_1 \oplus \cdots \oplus \naf p_{i-1} \oplus \naf p_{i+1} \oplus \cdots \oplus \naf p_n) & q \leftarrow \beta & q \leftarrow q \oplus q \quad
\end{eqnarray}
for all $i \in [1..n]$, where $q$ is a fresh atom.
The last two rules enforce $I(q) = 1$ whenever $I(\beta) > 0$, and $I(q) = 0$ otherwise.
For all $i \in [1..n]$, $I(q) = 0$ implies that the body of the first rule is interpreted as 0, and $I(q) = 1$ implies  $I(q \otimes \gamma) = I(\gamma)$, where $\gamma$ is $\beta \oplus \naf p_1 \oplus \cdots \oplus \naf p_{i-1} \oplus \naf p_{i+1} \oplus \cdots \oplus \naf p_n$.
Since the original rule is associated with the satisfaction of $\sum_{i \in [1..n]} I(p_i) - (n-1) \geq I(\beta)$, which is the case if and only if $I(p_i) \geq I(\beta) + \sum_{j \in [1..n], j \neq i} (1 - I(p_j))$, for all $i \in [1..n]$, this
translation preserves stable models for HCF programs.

The shift of $p_1 \veebar \cdots \veebar p_n \leftarrow \beta$ requires an even more advanced construction.
Notice first that since the program is HCF, we can order head atoms such that for every 
$1\le i<j\le n$, $p_i$ does not reach $p_j$ in $\G_\Pi$. Assume w.l.o.g. that one such ordering is given.
Then, the shift of this rule is the program containing the rules
\begin{eqnarray}
    p_i \leftarrow \beta \barwedge \naf q_1 \barwedge \cdots \barwedge \naf q_{i-1} \barwedge q_i \label{eq:veebar:1}\\
    q_i \leftarrow (p_i \veebar \cdots \veebar p_n) \otimes \naf(p_{i+1} \veebar \cdots \veebar p_n) & q_i \leftarrow q_i \oplus q_i & q_n \leftarrow 1 \label{eq:veebar:2}
\end{eqnarray}
for all $i \in [1..n]$, where each $q_i$ is a fresh atom.
Intuitively, (\ref{eq:veebar:2}) enforces $I(q_i) = 1$ whenever $I(p_i) > \max\{I(p_{i+1}), \ldots, I(p_n)\}$, and $I(q_i) = 0$ otherwise, with the exception of $I(q_n)$ which is always 1.
The rule (\ref{eq:veebar:1}) enforces that $I(p_i) \geq I(\beta)$ whenever $I(p_i) \geq \max\{I(p_1), \ldots, I(p_{i-1})\}$, and either $I(p_i) > \max\{I(p_{i+1}), \ldots, I(p_n)\}$ or $i = n$.
In the following, let $\shift(\Pi)$ denote the program obtained by shifting all rules of $\Pi$.

\begin{restatable}{theorem}{ThmShift}\label{thm:shift}
Let $\Pi$ be FASP program.
If $\Pi$ is HCF then $\Pi \equiv_{\At(\Pi)} \shift(\Pi)$.
\end{restatable}

\section{Translation into SMT}\label{sec:translation}

We now define a translation $\smt$ mapping $\Pi$ into a $\Sigma$-theory, where 
$\Sigma^C = \At(\Pi)$, and $\Sigma^V = \{x_p \mid p \in \At(\Pi)\}$.
The theory has two parts, \out and \inn, 
for producing a model and checking its minimality, respectively.
In more detail, $f \in \{\out, \inn\}$ is the following:
for $c \in [0,1]$, $f(c) = c$;
for $p \in \At(\Pi)$, $f(p)$ is $p$ if $f = \out$, and $x_p$ otherwise;
$f(\naf \alpha) = 1 - \out(\alpha)$;
$f(\alpha \oplus \beta) = \ite(t \leq 1, t, 1)$, where $t$ is $f(\alpha) + f(\beta)$;
$f(\alpha \otimes \beta) = \ite(t \geq 0, t, 0)$, where $t$ stands for $f(\alpha) + f(\beta) - 1$;
$f(\alpha \veebar \beta) = \ite(f(\alpha) \geq f(\beta), f(\alpha), f(\beta))$;
$f(\alpha \barwedge \beta) = \ite(f(\alpha) \leq f(\beta), f(\alpha), f(\beta))$;
$f(\alpha \leftarrow \beta) = f(\alpha) \geq f(\beta)$.
Note that propositional atoms are mapped to constants by \out, and to variables by \inn.
Moreover, negated expressions are always mapped by \out. 
Define $\smt(\Pi):=\{p \in [0,1] \mid p \in \At(\Pi)\} \cup \{\out(r) \mid r \in \Pi\} \cup \{\varphi_{\inn}\}$, where 
\begin{equation}
\varphi_\inn := \forall\{x_p \mid p \in \At(\Pi)\}.\bigwedge_{p \in \At(\Pi)} x_p \in [0,p] \wedge \bigwedge_{r \in \Pi} \inn(r) \rightarrow \bigwedge_{p \in \At(\Pi)} x_p = p.
\end{equation}

\begin{example}\label{ex:naive}
Consider the program $\Pi_2=\{p \leftarrow q \veebar \naf s,$ $q \oplus s \leftarrow \naf\naf p\}$.
The theory $\smt(\Pi_2)$ is
$\{p \in [0,1],$ $q \in [0,1],$ $s \in [0,1]\} \cup \{p \geq \ite(q \geq 1-s, q, 1-s),$
$\ite(q + s \leq 1, q+s, 1) \geq 1 - (1 - p)\} \cup \{\forall x_p.\forall x_q.\forall x_s.x_p \in [0,p] \wedge x_q \in [0,q] \wedge x_s \in [0,s] \wedge x_p \geq \ite(x_q \geq 1-s, x_q, 1-s) \wedge \ite(x_q + x_s \leq 1, x_q+x_s, 1) \geq 1 - (1 - p) \rightarrow x_p = p \wedge x_q = q \wedge x_s = s\}$.
Let \A be a $\Sigma$-structure such that $p^\A = q^\A = 1$ and $s^\A = 0$.
It can be checked that $\A \models \smt(\Pi_2)$.
Also note that $I(p) = I(q) = 1$ and $I(s) = 0$ implies $I \in \SM(\Pi_2)$.
\end{example}

For an interpretation $I$ of $\Pi$, let $\A_I$ be the one-to-one $\Sigma$-structure for $\smt(\Pi)$ such that $p^{\A_I} = I(p)$, for all $p \in \At(\Pi)$.

\begin{restatable}{theorem}{ThmSmt}\label{thm:smt}
Let $\Pi$ be a FASP program.
$I \in \SM(\Pi)$ if and only if $\A_I \models \smt(\Pi)$.
\end{restatable}

\subsection{Completion}\label{sec:completion}

A drawback of \smt is that it produces quantified theories, which are usually handled by incomplete 
heuristics in SMT solvers \cite{DBLP:conf/cav/GeM09}.
Structural properties of FASP programs may be exploited to obtain a more tailored translation that extends \emph{completion} \cite{DBLP:conf/adbt/Clark77} to the fuzzy case.
Completion is a translation into propositional theories used to compute stable models of acyclic ASP programs with atomic heads.
Intuitively, the models of the completion of a program $\Pi$ coincide with the \emph{supported models} of $\Pi$, 
i.e., those models $I$ with $I(p) = \max\{I(\beta) \mid p \leftarrow \beta \in \Pi\}$, for each $p \in \At(\Pi)$.
This notion was extended to FASP programs by \citeN{DBLP:journals/tplp/JanssenVSC12}, with fuzzy propositional theories as target framework.
We adapt it to produce
$\Sigma$-theories, for the $\Sigma$ defined before.

Let $\Pi$ be a program with atomic heads, and $p \in \At(\Pi)$.
We denote by $\heads(p,\Pi)$ the set of rules in $\Pi$ whose head is $p$, and by $\constraints(\Pi)$ the set of rules in $\Pi$ whose head 
is a numeric constant.
The completion of $\Pi$ is the $\Sigma$-theory:
\begin{equation}
    \begin{split}
        \comp(\Pi) :={} & \{p \in [0,1] \wedge p = \supp(p,\heads(p,\Pi)) \mid p \in \At(\Pi)\} \cup {} \\
            & \{\out(r) \mid r \in \constraints(\Pi)\},
     \end{split}
\end{equation}
where $\supp(p,\emptyset):=0$, and for $n \geq 1$,
$\supp(p,\{p \leftarrow \beta_i \mid i \in [1..n]\}):=\ite(\out(\beta_1) \geq t, \out(\beta_1), t)$,
where $t$ is $\supp(p,\{p \leftarrow \beta_i \mid i \in [2..n]\})$.
Basically, $\supp(p,\heads(p,\Pi))$ yields a term interpreted as $\max\{\out(\beta)^{\A_I} \mid p \leftarrow \beta \in \Pi\}$ by all $\Sigma$-structures $\A$.

\begin{example}
Since $\Pi_2$ in Example~\ref{ex:naive} is acyclic, $\Pi_2 \equiv_{\At(\Pi_2)} \shift(\Pi_2)$.
The theory $\comp(\shift(\Pi_2))$ is 
$\{p \in [0,1] \wedge p = \ite(q \geq 1-s, q, 1-s),$
$q \in [0,1] \wedge q = \ite(t_1 \geq 0, t_1, 0),$
$s \in [0,1] \wedge s = \ite(p-q \geq 0, p-q, 0)\}$, where $t_1$ is $(1 - (1-p)) + (1-s) -1$, and $t_2$ is $(1 - (1-p)) + (1-q) -1$.
\end{example}

Since $\smt(\Pi)$ and $\comp(\Pi)$ have the same constant symbols, $\A_I$ defines a one-to-one mapping between interpretations of $\Pi$ and $\Sigma$-structures of $\comp(\Pi)$.
An interesting question is whether correctness can be extended to HCF programs, for example by first shifting heads.
Notice that (\ref{eq:shift:otimes}) and (\ref{eq:veebar:2}) introduce rules of the form $q \leftarrow q \oplus q$ through the shift of $\otimes$ or $\veebar$, breaking acyclicity.
However, $q \leftarrow q \oplus q$ is a common pattern to force a Boolean interpretation of $q$, which can be encoded by integrality constraints in the theory.
The same observation applies to rules of the form $q \otimes q \leftarrow q$.
Define $\bool(\Pi):=\{p \leftarrow p \oplus p \in \Pi\} \cup \{p \otimes p \leftarrow p \in \Pi\}$,
and let $\bool^-(\Pi)$ be the program obtained from $\Pi \setminus \bool(\Pi)$ by performing the following operations for each $p \in \At(\bool(\Pi))$:
first, occurrences of $p$ in rule bodies are replaced by $b_p$, where $b_p$ is a fresh atom;
then, a choice rule $b_p \leftarrow \naf\naf b_p$ is added.
The refined completion is the following:
\begin{equation}
    \rcomp(\Pi) := \comp(\bool^-(\Pi)) \cup \{b_p = \ite(p > 0, 1, 0) \mid p \in \At(\bool(\Pi))\},
\end{equation}
and the associated $\Sigma$-structure $\A_I^r$ is such that $p^{\A_I^r} = I(p)$ for $p \in \At(\Pi)$, and $b_p^{\A_I^r}$ equals 1 if $I(p) > 0$, and 0 otherwise, for $p \in \At(\bool(\Pi))$.

\begin{restatable}{theorem}{ThmComp}\label{thm:comp}
Let $\Pi$ be a program such that $\Pi \setminus \bool(\Pi)$ is acyclic.
Then, $I \in \SM(\Pi)$ if and only if $\A_I^r \models \rcomp(\shift(\simp(\Pi)))$.
\end{restatable}

Note that in the above theorem $\simp$ and $\shift$ are only required because $\comp$ and $\rcomp$ are defined for normal programs.

\subsection{Ordered Completion}\label{sec:ordered}

Stable models of recursive programs do not coincide with supported models, making completion unsound.
To regain soundness, \emph{ordered completion} \cite{DBLP:journals/amai/Ben-EliyahuD94,DBLP:conf/ecai/Janhunen04,DBLP:journals/amai/Niemela08,DBLP:journals/ai/AsuncionLZZ12}
uses a notion of \emph{acyclic support}.
Let $\Pi$ be an ASP program with atomic heads.
$I$ is a stable model of $\Pi$ if and only if there exists a \emph{ranking} $r$ such that, for each $p \in I$, $I(p) = \max\{I(\beta) \mid p \leftarrow \beta \in \Pi,$ $r(p) = 1+\max(\{0\} \cup \{r(q) \mid q \in \pos(\beta)\})\}$ \cite{DBLP:conf/ecai/Janhunen04}.
This holds because the reduct $\Pi^I$ is also \naf-free, and thus its unique minimal model is the least fixpoint of the immediate consequence operator $\T_{\Pi^I}$, mapping interpretations $J$ to $\T_{\Pi^I}(J)$ where
$\T_{\Pi^I}(J)(p) := \max\{J(\beta) \mid p \leftarrow \beta \in \Pi^I\}$.
Since $J(\alpha \wedge \beta) \leq J(\alpha)$ and $J(\alpha \wedge \beta) \leq J(\beta)$, for all interpretations $J$, the limit is reached in $|\At(\Pi)|$ steps.
For FASP programs, however, the least fixpoint of $\T_{\Pi^I}$ is not reached within a linear number of applications \cite{DBLP:books/daglib/0035275}.
For example, $2^n$ applications are required for the program
$\{p \leftarrow p \oplus c\}$, for $c = 1/2^n$ and $n \geq 0$ \cite{DBLP:journals/ijar/BlondeelSVC14}.
On the other hand, for $\odot \in \{\barwedge,\otimes\}$ and all interpretations $J$, we have 
$J(\alpha \odot \beta) \leq J(\alpha)$ and $J(\alpha \odot \beta) \leq J(\beta)$.
The claim can thus be extended to the fuzzy case if recursion over $\oplus$ and $\veebar$ is disabled.

\begin{restatable}{lemma}{LemRank}\label{lem:rank}
Let $\Pi$ be such that $\Pi$ has atomic heads and non-recursive $\oplus,\veebar$ in rule bodies.
Let $I$ be an interpretation for $\Pi$.
The least fixpoint of $\T_{\Pi^I}$ is reached in $|\At(\Pi)|$ steps.
\end{restatable}

Ordered completion can be defined for this class of FASP programs.
Let $J$ be the least fixpoint of $\T_{\Pi^I}$.
The \emph{rank} of $p \in \At(\Pi)$ in $J$ is the step at which $J(p)$ is derived.
Let $r_p$ be a constant symbol expressing the rank of $p$.
Define $\rank(\emptyset):=1$, and $\rank(\{q_i \mid i \in [1..n]\}):=\ite(r_{q_1} \geq t, r_{q_1}, t)$ for $n \geq 1$,  where $t=\rank(\{q_i \mid i \in [2..n]\})$.
Also define $\osupp(p,\emptyset):=0$, and for $n \geq 1$, 
\begin{equation*}
\osupp(p,\{p \leftarrow \beta_i \mid i \in [1..n]\}):=\bigvee_{i \in [1..n]}(p = \out(\beta_i) \wedge r_p = 1 + \rank(\pos(\beta_i))).
\end{equation*}
The ordered completion of $\Pi$, denoted $\ocomp(\Pi)$, is the following theory:
\begin{equation}
    \comp(\Pi) \cup \{r_p\in [1..|\At(\Pi)|] \wedge p > 0 \rightarrow \osupp(p,\heads(p,\Pi)) \mid p \in \At(\Pi)\}.
\end{equation}
%
\begin{example}
The $\Sigma$-theory $\ocomp(\{p \leftarrow 0.1,$ $p \leftarrow q$, $q \leftarrow p\})$ is the following:
\begin{align*}
    & \{p \in [0,1] \wedge p = \ite(0.1 \geq q, 0.1, q)\} \cup \{q \in [0,1] \wedge q = p\} \\
    {}\cup{} & \{r_p \in [1..2] \wedge p > 0 \rightarrow (p = 0.1 \wedge r_p = 1 + 0) \vee (p = q \wedge r_p = 1 + r_q)\} \\
    {}\cup{} & \{r_q \in [1..2] \wedge q > 0 \rightarrow q = p \wedge r_q = 1 + r_p)\}.
\end{align*}
The theory is satisfied by \A if $p^\A = q^\A = 0.1$, $r_p^\A = 1$, and $r_q^\A = 2$.
\end{example}
%

The correctness of \ocomp, provided that $\Pi$ satisfies the conditions of Lemma~\ref{lem:rank}, is proved by the following mappings:
for $I \in \SM(\Pi)$, let $\A^o_I$ be the $\Sigma$-model for $\ocomp(\Pi)$ such that $p^{\A^o_I} = I(p)$ and $r_p^{\A^o_I}$ is the rank of $p$ in $I$, for all $p \in \At(\Pi)$;
for $\A$ such that $\A \models \ocomp(\Pi)$, let $I_\A$ be the interpretation for $\Pi$ such that $I_\A(p) = p^\A$, for all $p \in \At(\Pi)$.

\begin{restatable}{theorem}{ThmOcomp}\label{thm:ocomp}
Let $\Pi$ be an HCF program with non-recursive $\oplus$ in rule bodies, and whose head connectives are $\barwedge,\oplus$.
If $I \in \SM(\Pi)$ then $\A^o_I \models \ocomp(\shift(\simp(\Pi)))$.
Dually, if $\A \models \ocomp(\shift(\simp(\Pi)))$ then $I_\A \in \SM(\Pi)$.
\end{restatable}

The above theorem does not apply in case of recursive $\oplus$ in rule bodies.
For example, $\{p \leftarrow p \oplus 0.1\}$ has a unique stable model assigning $1$ to $p$, while its ordered completion is
the following $\Sigma$-theory with no $\Sigma$-model:
$\{p \in [0,1] \wedge p = \ite(p + 0.1 \leq 1, p + 0.1, 1)\} \cup \{r_p \in [1..1] \wedge p > 0 \rightarrow p = \ite(p + 0.1 \leq 1, p + 0.1, 1) \wedge r_p = 1 + r_p\}$.

\section{Implementation and Experiment}\label{sec:experiment}

We implemented the translations from Section~\ref{sec:simplification} in the new FASP solver \textsc{fasp2smt}.
\textsc{fasp2smt} is written in \textsc{python}, and uses \textsc{gringo} \cite{DBLP:conf/lpnmr/GebserKKS11} to obtain a ground representation of the input program, and \textsc{z3} \mbox{\cite{DBLP:conf/tacas/MouraB08}} to solve SMT instances encoding ground programs.
The output of \textsc{gringo} encodes a propositional program, say $\Pi$, that is conformant with the syntax in Section~\ref{sec:background}.
The components of $\Pi$ are computed, and the structure of the program is analyzed.
If $\Pi \setminus \bool(\Pi)$ is acyclic, $\rcomp(\shift(\simp(\Pi)))$ is built.
If $\Pi$ is HCF with non-recursive $\oplus$ in rule bodies, and only $\barwedge$ and $\oplus$ in rule heads, then $\ocomp(\shift(\simp(\Pi)))$ is built.
In all other cases, $\smt(\simp(\Pi))$ is built.
The built theory is fed into \textsc{z3}, and either a stable model or the string \textsc{incoherent} is reported.

The performance of \textsc{fasp2smt} was assessed on instances of a benchmark used to evaluate the FASP solver \textsc{ffasp} \cite{DBLP:conf/ecai/MushthofaSC14}.
The benchmark comprises two (synthetic) problems, the fuzzy versions of \emph{Graph Coloring} and \emph{Hamiltonian Path}, originally considered by \citeN{DBLP:journals/tplp/AlvianoP13}.
In Graph Coloring edges of an input graph are associated with truth degrees, and each vertex $x$ is non-deterministically colored with a shadow of gray, i.e., truth degree 1 is distributed among the atoms $\mathit{black}_x$ and $\mathit{white}_x$.
The truth degree of each edge $xy$, say $d$, enforces $d \otimes \mathit{black}_x \otimes \mathit{black}_y = 0$ and $d \otimes \mathit{white}_x \otimes \mathit{white}_y = 0$, i.e., adjacent vertices must be colored with sufficiently different shadows of gray.
Similarly, in Hamiltonian Path vertices and edges of an input graph are associated with truth degrees, and Boolean connectives are replaced by \Luka connectives in the usual ASP encoding.
The truth degree of each edge $xy$, say $d$, is non-deterministically distributed among the atoms $\mathit{in}_{xy}$ and $\mathit{out}_{xy}$.
Reaching a vertex $y$ from the initial vertex $x$ via an edge $xy$ guarantees that $y$ is reached with truth degree $\mathit{in}_{xy}$.
Reaching a third vertex $z$ via an edge $yz$, instead, guarantees that $z$ is reached with truth degree $\mathit{in}_{xy} \otimes \mathit{in}_{yz}$.
In other words, the more uncertain is the selection of an edge $xy$, the more uncertain is the membership of $y$ in the selected path, which in turn implies an even more uncertain membership of any $z$ reached by an edge $yz$.
In the original encodings, \Luka disjunction was used to guess (fuzzy) membership of elements in one of two sets.
For example, Hamiltonian Path used a rule of the form
$\mathit{in}(X,Y) \oplus \mathit{out}(X,Y) \leftarrow \mathit{edge}(X,Y)$, which was shifted and replaced by
$\mathit{in}(X,Y) \leftarrow \mathit{edge}(X,Y) \otimes \naf\mathit{out}(X,Y)$ and
$\mathit{out}(X,Y) \leftarrow \mathit{edge}(X,Y) \otimes \naf\mathit{in}(X,Y)$ by \citeANP{DBLP:journals/tplp/AlvianoP13}.
In fact, in 2013 the focus was on FASP programs with atomic heads and only $\otimes$ in rule bodies, and the shift of $\oplus$ for these programs was implicit in the work of \citeN{DBLP:journals/ijar/BlondeelSVC14}.
Since our focus is now on a more general setting, the original encodings were restored, even if it is clear that \textsc{fasp2smt} shifts such programs by itself.
In fact, Graph Coloring is recognized as acyclic, and Hamiltonian Path as HCF with no $\oplus$ in rule bodies.
It turns out that \textsc{fasp2smt} uses completion for Graph Coloring, and ordered completion for Hamiltonian Path.
The experiment was run on an Intel Xeon CPU 2.4 GHz with 16 GB of RAM.
CPU and memory usage were limited to 600 seconds and 15 GB, respectively.
\textsc{fasp2smt} and \textsc{ffasp} were tested with their default settings, and the performance was measured by \textsc{pyrunlim} (\url{http://alviano.net/software/pyrunlim/}), the tool used in the last ASP Competitions \cite{DBLP:conf/lpnmr/AlvianoCCDDIKKOPPRRSSSWX13,DBLP:journals/corr/CalimeriGMR14}.

\begin{table}[b]
    \caption{Performance of \textsc{fasp2smt} and \textsc{ffasp} (average execution time in seconds; average memory consumption in MB).}\label{tab:experiment}

    \begin{tabular}{rrrrrrrrrrrrrrrrrr}
    \toprule
        &&& \multicolumn{3}{c}{\textsc{fasp2smt}} &  \multicolumn{3}{c}{\textsc{ffasp}} & \multicolumn{3}{c}{\textsc{ffasp} (shifted enc.)} \\
        \cmidrule{4-6}\cmidrule{7-9}\cmidrule{10-12}
    	&	\bf den	&	\bf inst	&	\bf sol	&	\bf time	&	\bf mem	&	\bf sol	&	\bf time	&	\bf mem &	\bf sol	&	\bf time	&	\bf mem \\
	\cmidrule{1-12}
    \parbox[t]{2mm}{\multirow{5}{*}{\rotatebox[origin=c]{90}{\bf graph--col}}} 
	&	20	&	6	&	6	&	94.0	&	174	&	6	&	5.3	&	302	&	6	&	1.5	&	69 \\
	&	40	&	6	&	6	&	102.4	&	178	&	6	&	19.8	&	1112	&	6	&	5.3	&	181 \\
	&	60	&	6	&	6	&	107.6	&	180	&	6	&	46.7	&	2472	&	6	&	11.8	&	342 \\
	&	80	&	6	&	6	&	111.1	&	181	&	6	&	90.1	&	4420	&	6	&	21.0	&	550 \\
	&	100	&	6	&	6	&	111.7	&	181	&	6	&	151.9	&	7025	&	6	&	33.6	&	812 \\
    \rule{0pt}{2.5ex}
    \parbox[t]{2mm}{\multirow{10}{*}{\rotatebox[origin=c]{90}{\bf ham--path}}}
	&	20	&	10	&	10	&	1.7	&	25	&	10	&	17.3	&	410	&	10	&	3.5	&	101 \\
	&	40	&	10	&	10	&	1.8	&	25	&	10	&	20.3	&	462	&	10	&	2.3	&	105 \\
	&	60	&	10	&	10	&	2.1	&	25	&	10	&	13.2	&	481	&	10	&	2.0	&	107 \\
	&	80	&	10	&	10	&	2.4	&	25	&	10	&	32.9	&	868	&	10	&	3.9	&	188 \\
	&	100	&	10	&	10	&	2.1	&	25	&	10	&	69.0	&	1385	&	10	&	6.5	&	323 \\
	&	120	&	10	&	10	&	2.0	&	25	&	10	&	125.5	&	2042	&	10	&	10.5	&	475 \\
	&	140	&	10	&	10	&	1.9	&	25	&	10	&	176.8	&	2821	&	10	&	14.7	&	669 \\
	&	160	&	10	&	10	&	2.2	&	25	&	9	&	139.6	&	3769	&	10	&	20.8	&	960 \\
	&	180	&	10	&	10	&	2.4	&	26	&	8	&	203.1	&	4914	&	10	&	28.9	&	1270 \\
	\bottomrule
    \end{tabular}
\end{table}

The results are reported in Table~\ref{tab:experiment}.
Instances are grouped according to the granularity of numeric constants, where instances with $\textbf{den} = d$ are characterized by numeric constants of the form $n/d$.
There are 6 instances of Graph Coloring and 10 of Hamiltonian Path in each group.
All instances of Graph Coloring are coherent, while there is an average of 4 incoherent instances in each group of Hamiltonian Path.
All instances are solved by \textsc{fasp2smt} (column \textbf{sol}), and the granularity of numeric constants does not really impact on execution time and memory consumption.
The performance is particularly good for Hamiltonian Path, while \textsc{ffasp} is faster than \textsc{fasp2smt} in Graph Coloring for numeric constants of limited granularity.
The performance of \textsc{ffasp} deteriorates when the granularity of numeric constants increases, and 6 timeouts are reported for the largest instances of Hamiltonian Path.
Another strength of \textsc{fasp2smt} is the limited memory consumption compared to \textsc{ffasp}.
If we decrease the memory limit to 3 GB, \textsc{ffasp} runs out of memory on 12 instances of Graph Coloring and 34 instances of Hamiltonian Path, while \textsc{fasp2smt} still succeeds in all instances.
For the sake of completeness, manually shifted encodings were also tested.
The performance of \textsc{fasp2smt} did not change, while \textsc{ffasp} improves considerably, especially regarding memory consumption.
We also tested 180 instances (not reported in Table~\ref{tab:experiment}) of two simple problems called \emph{Stratified} and \emph{Odd Cycle} \cite{DBLP:journals/tplp/AlvianoP13,DBLP:conf/ecai/MushthofaSC14}, which both \textsc{fasp2smt} and \textsc{ffasp} solve in less than 1 second.

The main picture resulting from the experimental analysis is that \textsc{fasp2smt} is slower than \textsc{ffasp} in Graph Coloring, but it is faster in Hamiltonian Path.
The reason for these different behaviors can be explained by the fact that all tested instances of Graph Coloring are coherent, while incoherent instances are also present among those tested for Hamiltonian Path.
To confirm such an intuition, we tested the simple program $\{p \oplus q \leftarrow 1, 0 \leftarrow p \oplus q\}$.
Its incoherence is proved instantaneously by \textsc{fasp2smt}, while \textsc{ffasp} requires 71.8 seconds and 446 MB of memory (8.3 seconds and 96 MB of memory if the program is manually shifted).

\section{Conclusions}

SMT proved to be a reasonable target language to compute fuzzy answer sets efficiently.
In fact, when structural properties of the evaluated programs are taken into account, efficiently evaluable theories are produced by \textsc{fasp2smt}.
This is the case for acyclic programs, for which completion can be used, as well as for HCF programs with only $\oplus$ in rule heads and no recursive $\oplus$ in rule bodies, for which ordered completion is proposed.
Moreover, common patterns to \emph{crispify} atoms, which would introduce recursive $\oplus$ in rule bodies, are possibly replaced by integrality constraints.
The performance of \textsc{fasp2smt} was compared with \textsc{ffasp}, which performs multiple calls to an ASP solver.
An advantage of \textsc{fasp2smt} is that, contrary to \textsc{ffasp}, its performance is not affected by the approximation used to represent truth degrees in the input program.
On the other hand, \textsc{ffasp} is currently faster than \textsc{fasp2smt} for instances having a stable model with truth degrees in $\Q_k$, for some small $k$, which however cannot be determined a priori.
Such a $k$ does not exist for incoherent instances, and indeed in this case \textsc{fasp2smt} significantly overcomes \textsc{ffasp}.
It is also important to note that in general the amount of memory required by \textsc{fasp2smt} is negligible compared to \textsc{ffasp}.
Future work will evaluate the possibility to extend the approximation operators by \citeN{DBLP:journals/tplp/AlvianoP13} to the broader language considered in this paper, with the aim of identifing classes of programs for which the fixpoints are reached within a linear number of applications.

\section*{Acknowledgement}
Mario Alviano was partially supported by MIUR within project ``SI-LAB BA2\-KNOW  -- Business Analitycs to Know'', by Regione Calabria, POR Calabria FESR 2007-2013, within projects ``ITravel PLUS'' and ``KnowRex'', by the National Group for Scientific Computation (GNCS-INDAM), and by Finanziamento Giovani Ricercatori UNICAL.
Rafael Pe\~naloza was partially supported by the DFG within the Cluster of Excellence `cfAED;' this work was 
developed while still being affiliated with TU Dresden and the Center for Advancing Electronics Dresden, Germany.

\bibliographystyle{acmtrans}
\bibliography{bibtex}

\begin{thebibliography}{}

\bibitem[\protect\citeauthoryear{Akbarpour and Paulson}{Akbarpour and
  Paulson}{2010}]{DBLP:journals/jar/AkbarpourP10}
{\sc Akbarpour, B.} {\sc and} {\sc Paulson, L.~C.} 2010.
\newblock Metitarski: An automatic theorem prover for real-valued special
  functions.
\newblock {\em J. Autom. Reasoning\/}~{\em 44,\/}~3, 175--205.

\bibitem[\protect\citeauthoryear{Alviano, Calimeri, Charwat, Dao-Tran, Dodaro,
  Ianni, Krennwallner, Kronegger, Oetsch, Pfandler, P{\"u}hrer, Redl, Ricca,
  Schneider, Schwengerer, Spendier, Wallner, and Xiao}{Alviano
  et~al\mbox{.}}{2013}]{DBLP:conf/lpnmr/AlvianoCCDDIKKOPPRRSSSWX13}
{\sc Alviano, M.}, {\sc Calimeri, F.}, {\sc Charwat, G.}, {\sc Dao-Tran, M.},
  {\sc Dodaro, C.}, {\sc Ianni, G.}, {\sc Krennwallner, T.}, {\sc Kronegger,
  M.}, {\sc Oetsch, J.}, {\sc Pfandler, A.}, {\sc P{\"u}hrer, J.}, {\sc Redl,
  C.}, {\sc Ricca, F.}, {\sc Schneider, P.}, {\sc Schwengerer, M.}, {\sc
  Spendier, L.~K.}, {\sc Wallner, J.~P.}, {\sc and} {\sc Xiao, G.} 2013.
\newblock The fourth answer set programming competition: Preliminary report.
\newblock In {\em LPNMR}, {P.~Cabalar} {and} {T.~C. Son}, Eds. LNCS. 42--53.

\bibitem[\protect\citeauthoryear{Alviano, Dodaro, Faber, Leone, and
  Ricca}{Alviano et~al\mbox{.}}{2013}]{DBLP:conf/lpnmr/AlvianoDFLR13}
{\sc Alviano, M.}, {\sc Dodaro, C.}, {\sc Faber, W.}, {\sc Leone, N.}, {\sc
  and} {\sc Ricca, F.} 2013.
\newblock {WASP:} {A} native {ASP} solver based on constraint learning.
\newblock In {\em Logic Programming and Nonmonotonic Reasoning, 12th
  International Conference, {LPNMR} 2013, Corunna, Spain, September 15-19,
  2013. Proceedings}, {P.~Cabalar} {and} {T.~C. Son}, Eds. Lecture Notes in
  Computer Science, vol. 8148. Springer, 54--66.

\bibitem[\protect\citeauthoryear{Alviano, Faber, Leone, Perri, Pfeifer, and
  Terracina}{Alviano et~al\mbox{.}}{2010}]{DBLP:conf/datalog/AlvianoFLPPT10}
{\sc Alviano, M.}, {\sc Faber, W.}, {\sc Leone, N.}, {\sc Perri, S.}, {\sc
  Pfeifer, G.}, {\sc and} {\sc Terracina, G.} 2010.
\newblock The disjunctive datalog system {DLV}.
\newblock In {\em Datalog Reloaded - First International Workshop, Datalog
  2010, Oxford, UK, March 16-19, 2010. Revised Selected Papers}, {O.~de~Moor},
  {G.~Gottlob}, {T.~Furche}, {and} {A.~J. Sellers}, Eds. Lecture Notes in
  Computer Science, vol. 6702. Springer, 282--301.

\bibitem[\protect\citeauthoryear{Alviano and Pe{\~{n}}aloza}{Alviano and
  Pe{\~{n}}aloza}{2013}]{DBLP:journals/tplp/AlvianoP13}
{\sc Alviano, M.} {\sc and} {\sc Pe{\~{n}}aloza, R.} 2013.
\newblock Fuzzy answer sets approximations.
\newblock {\em {TPLP}\/}~{\em 13,\/}~4-5, 753--767.

\bibitem[\protect\citeauthoryear{Asuncion, Lin, Zhang, and Zhou}{Asuncion
  et~al\mbox{.}}{2012}]{DBLP:journals/ai/AsuncionLZZ12}
{\sc Asuncion, V.}, {\sc Lin, F.}, {\sc Zhang, Y.}, {\sc and} {\sc Zhou, Y.}
  2012.
\newblock Ordered completion for first-order logic programs on finite
  structures.
\newblock {\em Artif. Intell.\/}~{\em 177-179}, 1--24.

\bibitem[\protect\citeauthoryear{Baral}{Baral}{2003}]{bara-2002}
{\sc Baral, C.} 2003.
\newblock {\em {Knowledge Representation, Reasoning and Declarative Problem
  Solving}}.
\newblock Cambridge University Press.

\bibitem[\protect\citeauthoryear{Barrett, Sebastiani, Seshia, and
  Tinelli}{Barrett et~al\mbox{.}}{2009}]{DBLP:series/faia/BarrettSST09}
{\sc Barrett, C.~W.}, {\sc Sebastiani, R.}, {\sc Seshia, S.~A.}, {\sc and} {\sc
  Tinelli, C.} 2009.
\newblock Satisfiability modulo theories.
\newblock In {\em Handbook of Satisfiability}, {A.~Biere}, {M.~Heule}, {H.~van
  Maaren}, {and} {T.~Walsh}, Eds. Frontiers in Artificial Intelligence and
  Applications, vol. 185. {IOS} Press, 825--885.

\bibitem[\protect\citeauthoryear{Ben{-}Eliyahu and Dechter}{Ben{-}Eliyahu and
  Dechter}{1994}]{DBLP:journals/amai/Ben-EliyahuD94}
{\sc Ben{-}Eliyahu, R.} {\sc and} {\sc Dechter, R.} 1994.
\newblock Propositional semantics for disjunctive logic programs.
\newblock {\em Ann. Math. Artif. Intell.\/}~{\em 12,\/}~1-2, 53--87.

\bibitem[\protect\citeauthoryear{Blondeel, Schockaert, Vermeir, and
  Cock}{Blondeel et~al\mbox{.}}{2014}]{DBLP:journals/ijar/BlondeelSVC14}
{\sc Blondeel, M.}, {\sc Schockaert, S.}, {\sc Vermeir, D.}, {\sc and} {\sc
  Cock, M.~D.} 2014.
\newblock Complexity of fuzzy answer set programming under {\l}ukasiewicz
  semantics.
\newblock {\em Int. J. Approx. Reasoning\/}~{\em 55,\/}~9, 1971--2003.

\bibitem[\protect\citeauthoryear{Calimeri, Gebser, Maratea, and Ricca}{Calimeri
  et~al\mbox{.}}{2014}]{DBLP:journals/corr/CalimeriGMR14}
{\sc Calimeri, F.}, {\sc Gebser, M.}, {\sc Maratea, M.}, {\sc and} {\sc Ricca,
  F.} 2014.
\newblock The design of the fifth answer set programming competition.
\newblock {\em CoRR\/}~{\em abs/1405.3710}.

\bibitem[\protect\citeauthoryear{Clark}{Clark}{1977}]{DBLP:conf/adbt/Clark77}
{\sc Clark, K.~L.} 1977.
\newblock Negation as failure.
\newblock In {\em Logic and Data Bases}. 293--322.

\bibitem[\protect\citeauthoryear{de~Moura and Bj{\o}rner}{de~Moura and
  Bj{\o}rner}{2008}]{DBLP:conf/tacas/MouraB08}
{\sc de~Moura, L.~M.} {\sc and} {\sc Bj{\o}rner, N.} 2008.
\newblock {Z3:} an efficient {SMT} solver.
\newblock In {\em Tools and Algorithms for the Construction and Analysis of
  Systems, 14th International Conference, {TACAS} 2008, Budapest, Hungary,
  March 29-April 6, 2008. Proceedings}, {C.~R. Ramakrishnan} {and} {J.~Rehof},
  Eds. Lecture Notes in Computer Science, vol. 4963. Springer, 337--340.

\bibitem[\protect\citeauthoryear{Delgrande, Schaub, Tompits, and
  Woltran}{Delgrande et~al\mbox{.}}{2008}]{DBLP:conf/kr/DelgrandeSTW08}
{\sc Delgrande, J.~P.}, {\sc Schaub, T.}, {\sc Tompits, H.}, {\sc and} {\sc
  Woltran, S.} 2008.
\newblock Belief revision of logic programs under answer set semantics.
\newblock In {\em Principles of Knowledge Representation and Reasoning:
  Proceedings of the Eleventh International Conference, KR 2008, Sydney,
  Australia, September 16-19, 2008}, {G.~Brewka} {and} {J.~Lang}, Eds.
  411--421.

\bibitem[\protect\citeauthoryear{Eiter, Fink, and Woltran}{Eiter
  et~al\mbox{.}}{2007}]{DBLP:journals/tocl/EiterFW07}
{\sc Eiter, T.}, {\sc Fink, M.}, {\sc and} {\sc Woltran, S.} 2007.
\newblock Semantical characterizations and complexity of equivalences in answer
  set programming.
\newblock {\em {ACM} Trans. Comput. Log.\/}~{\em 8,\/}~3.

\bibitem[\protect\citeauthoryear{Eiter and Gottlob}{Eiter and
  Gottlob}{1995}]{DBLP:journals/amai/EiterG95}
{\sc Eiter, T.} {\sc and} {\sc Gottlob, G.} 1995.
\newblock On the computational cost of disjunctive logic programming:
  Propositional case.
\newblock {\em Ann. Math. Artif. Intell.\/}~{\em 15,\/}~3-4, 289--323.

\bibitem[\protect\citeauthoryear{Ge and de~Moura}{Ge and
  de~Moura}{2009}]{DBLP:conf/cav/GeM09}
{\sc Ge, Y.} {\sc and} {\sc de~Moura, L.~M.} 2009.
\newblock Complete instantiation for quantified formulas in satisfiabiliby
  modulo theories.
\newblock In {\em Computer Aided Verification, 21st International Conference,
  {CAV} 2009, Grenoble, France, June 26 - July 2, 2009. Proceedings},
  {A.~Bouajjani} {and} {O.~Maler}, Eds. Lecture Notes in Computer Science, vol.
  5643. Springer, 306--320.

\bibitem[\protect\citeauthoryear{Gebser, Kaminski, K{\"{o}}nig, and
  Schaub}{Gebser et~al\mbox{.}}{2011}]{DBLP:conf/lpnmr/GebserKKS11}
{\sc Gebser, M.}, {\sc Kaminski, R.}, {\sc K{\"{o}}nig, A.}, {\sc and} {\sc
  Schaub, T.} 2011.
\newblock Advances in \emph{gringo} series 3.
\newblock In {\em Logic Programming and Nonmonotonic Reasoning - 11th
  International Conference, {LPNMR} 2011, Vancouver, Canada, May 16-19, 2011.
  Proceedings}, {J.~P. Delgrande} {and} {W.~Faber}, Eds. Lecture Notes in
  Computer Science, vol. 6645. Springer, 345--351.

\bibitem[\protect\citeauthoryear{Gebser, Kaufmann, and Schaub}{Gebser
  et~al\mbox{.}}{2012}]{DBLP:journals/ai/GebserKS12}
{\sc Gebser, M.}, {\sc Kaufmann, B.}, {\sc and} {\sc Schaub, T.} 2012.
\newblock Conflict-driven answer set solving: From theory to practice.
\newblock {\em Artif. Intell.\/}~{\em 187}, 52--89.

\bibitem[\protect\citeauthoryear{Gelfond and Lifschitz}{Gelfond and
  Lifschitz}{1991}]{DBLP:journals/ngc/GelfondL91}
{\sc Gelfond, M.} {\sc and} {\sc Lifschitz, V.} 1991.
\newblock Classical negation in logic programs and disjunctive databases.
\newblock {\em New Generation Comput.\/}~{\em 9,\/}~3/4, 365--386.

\bibitem[\protect\citeauthoryear{Janhunen}{Janhunen}{2004}]{DBLP:conf/ecai/Janhunen04}
{\sc Janhunen, T.} 2004.
\newblock Representing normal programs with clauses.
\newblock In {\em Proceedings of the 16th Eureopean Conference on Artificial
  Intelligence, ECAI'2004, Valencia, Spain, August 22-27, 2004}, {R.~L.
  de~M{\'{a}}ntaras} {and} {L.~Saitta}, Eds. {IOS} Press, 358--362.

\bibitem[\protect\citeauthoryear{Janssen, Schockaert, Vermeir, and
  Cock}{Janssen et~al\mbox{.}}{2012}]{DBLP:books/daglib/0035275}
{\sc Janssen, J.}, {\sc Schockaert, S.}, {\sc Vermeir, D.}, {\sc and} {\sc
  Cock, M.~D.} 2012.
\newblock {\em Answer Set Programming for Continuous Domains - {A} Fuzzy Logic
  Approach}. Atlantis Computational Intelligence Systems, vol.~5.
\newblock Atlantis Press.

\bibitem[\protect\citeauthoryear{Janssen, Vermeir, Schockaert, and
  Cock}{Janssen et~al\mbox{.}}{2012}]{DBLP:journals/tplp/JanssenVSC12}
{\sc Janssen, J.}, {\sc Vermeir, D.}, {\sc Schockaert, S.}, {\sc and} {\sc
  Cock, M.~D.} 2012.
\newblock Reducing fuzzy answer set programming to model finding in fuzzy
  logics.
\newblock {\em {TPLP}\/}~{\em 12,\/}~6, 811--842.

\bibitem[\protect\citeauthoryear{Lee and Wang}{Lee and
  Wang}{2014}]{DBLP:conf/jelia/LeeW14}
{\sc Lee, J.} {\sc and} {\sc Wang, Y.} 2014.
\newblock Stable models of fuzzy propositional formulas.
\newblock In {\em Logics in Artificial Intelligence - 14th European Conference,
  {JELIA} 2014, Funchal, Madeira, Portugal, September 24-26, 2014.
  Proceedings}, {E.~Ferm{\'{e}}} {and} {J.~Leite}, Eds. Lecture Notes in
  Computer Science, vol. 8761. Springer, 326--339.

\bibitem[\protect\citeauthoryear{Lierler and Maratea}{Lierler and
  Maratea}{2004}]{DBLP:conf/lpnmr/LierlerM04}
{\sc Lierler, Y.} {\sc and} {\sc Maratea, M.} 2004.
\newblock Cmodels-2: Sat-based answer set solver enhanced to non-tight
  programs.
\newblock In {\em Logic Programming and Nonmonotonic Reasoning, 7th
  International Conference, {LPNMR} 2004, Fort Lauderdale, FL, USA, January
  6-8, 2004, Proceedings}, {V.~Lifschitz} {and} {I.~Niemel{\"{a}}}, Eds.
  Lecture Notes in Computer Science, vol. 2923. Springer, 346--350.

\bibitem[\protect\citeauthoryear{Lin and You}{Lin and
  You}{2002}]{DBLP:journals/ai/LinY02}
{\sc Lin, F.} {\sc and} {\sc You, J.-H.} 2002.
\newblock Abduction in logic programming: A new definition and an abductive
  procedure based on rewriting.
\newblock {\em Artificial Intelligence\/}~{\em 140,\/}~1/2, 175--205.

\bibitem[\protect\citeauthoryear{Marek and Remmel}{Marek and
  Remmel}{2004}]{DBLP:conf/nmr/MarekR04}
{\sc Marek, V.~W.} {\sc and} {\sc Remmel, J.~B.} 2004.
\newblock Answer set programming with default logic.
\newblock In {\em 10th International Workshop on Non-Monotonic Reasoning (NMR
  2004), Whistler, Canada, June 6-8, 2004, Proceedings}, {J.~P. Delgrande}
  {and} {T.~Schaub}, Eds. 276--284.

\bibitem[\protect\citeauthoryear{Marek and Truszczy{\'n}ski}{Marek and
  Truszczy{\'n}ski}{1999}]{mare-trus-99}
{\sc Marek, V.~W.} {\sc and} {\sc Truszczy{\'n}ski, M.} 1999.
\newblock {Stable Models and an Alternative Logic Programming Paradigm}.
\newblock In {\em {The Logic Programming Paradigm -- A 25-Year Perspective}},
  {K.~R. Apt}, {V.~W. Marek}, {M.~Truszczy{\'n}ski}, {and} {D.~S. Warren}, Eds.
  Springer Verlag, 375--398.

\bibitem[\protect\citeauthoryear{Mushthofa, Schockaert, and Cock}{Mushthofa
  et~al\mbox{.}}{2014}]{DBLP:conf/ecai/MushthofaSC14}
{\sc Mushthofa, M.}, {\sc Schockaert, S.}, {\sc and} {\sc Cock, M.~D.} 2014.
\newblock A finite-valued solver for disjunctive fuzzy answer set programs.
\newblock In {\em {ECAI} 2014 - 21st European Conference on Artificial
  Intelligence, 18-22 August 2014, Prague, Czech Republic}, {T.~Schaub},
  {G.~Friedrich}, {and} {B.~O'Sullivan}, Eds. Frontiers in Artificial
  Intelligence and Applications, vol. 263. {IOS} Press, 645--650.

\bibitem[\protect\citeauthoryear{Niemel{\"{a}}}{Niemel{\"{a}}}{1999}]{DBLP:journals/amai/Niemela99}
{\sc Niemel{\"{a}}, I.} 1999.
\newblock Logic programs with stable model semantics as a constraint
  programming paradigm.
\newblock {\em Ann. Math. Artif. Intell.\/}~{\em 25,\/}~3-4, 241--273.

\bibitem[\protect\citeauthoryear{Niemel{\"{a}}}{Niemel{\"{a}}}{2008}]{DBLP:journals/amai/Niemela08}
{\sc Niemel{\"{a}}, I.} 2008.
\newblock Stable models and difference logic.
\newblock {\em Ann. Math. Artif. Intell.\/}~{\em 53,\/}~1-4, 313--329.

\bibitem[\protect\citeauthoryear{Nieuwenborgh, Cock, and Vermeir}{Nieuwenborgh
  et~al\mbox{.}}{2007}]{DBLP:journals/amai/NieuwenborghCV07}
{\sc Nieuwenborgh, D.~V.}, {\sc Cock, M.~D.}, {\sc and} {\sc Vermeir, D.} 2007.
\newblock An introduction to fuzzy answer set programming.
\newblock {\em Ann. Math. Artif. Intell.\/}~{\em 50,\/}~3-4, 363--388.

\bibitem[\protect\citeauthoryear{Ratschan}{Ratschan}{2006}]{DBLP:journals/tocl/Ratschan06}
{\sc Ratschan, S.} 2006.
\newblock Efficient solving of quantified inequality constraints over the real
  numbers.
\newblock {\em {ACM} Trans. Comput. Log.\/}~{\em 7,\/}~4, 723--748.

\end{thebibliography}

\label{lastpage}

\clearpage
\appendix

\section{Proofs}

\PropSimp*
\begin{proof}
Since each rule is rewritten independently, we can prove $\Pi \equiv_{\At(\Pi)} (\Pi \setminus \{r\}) \cup \simp(\{r\})$, where $r$ is some rule in $\Pi$.
We use structural induction on $r$.
The base case, i.e., $r$ is of the form $\alpha \leftarrow \beta$ with $\alpha \in \B$ and $\beta \in \B$, is trivial because $\simp(\{\alpha \leftarrow \beta\}) = \{\alpha \leftarrow \beta\}$.
Now, consider $r$ of the form $\alpha \leftarrow \naf\beta$.
We have to show $\Pi \equiv_{\At(\Pi)} \Pi'$, where $\Pi' := (\Pi \setminus \{r\}) \cup \{\alpha \leftarrow \naf p,$ $p \leftarrow \beta\}$.
For $I \in \SM(\Pi)$, define $I'$ such that $I'(p) := I(\beta)$, and $I'(q) := I(q)$ for all $q \in \At(\Pi)$.
We have that $I' \in \SM(\Pi')$.
Moreover, for any $J \in \SM(\Pi')$ it holds that $J(p) = J(\beta)$ because the only head occurrence of $p$ in $\Pi'$ is in $p \leftarrow \beta$.
It turns out that $J \cap \At(\Pi)$ belongs to $\SM(\Pi)$.
The remaining cases are given in \cite{DBLP:conf/ecai/MushthofaSC14}.
\end{proof}

\ThmHard*
\begin{proof}
We start by giving the common properties that will be used to prove each part of the theorem.
We reduce the satisfiability problem for 2-QBF$_\exists$ formulas to FASP coherence testing.
Let $\phi$ be  $\exists x_1,\ldots,x_m \forall x_{m+1},\ldots,x_n\ \bigvee_{i=1}^k L_{k,1} \wedge L_{k,2} \wedge L_{k,3}$, where $n > m \geq 1$, $k \geq 1$.
For each $\odot \in \{\veebar,\oplus,\otimes\}$, our aim is to build a FASP program $\Pi_\phi^\odot$ such that $\phi$ is satisfiable if and only if $\Pi_\phi^\odot$ is coherent.

In the construction of $\Pi_\phi^\odot$ we use the mapping $\sigma$ such that $\sigma(x_i) := x_i^T$, and $\sigma(\neg x_i) := x_i^F$, for all $i \in [1..n]$.
Moreover, $\Pi_\phi^\odot$ will have atoms $\mathit{sat}$, and $x_i^T,x_i^F$ for all $i \in [1..n]$, and its models will satisfy the following properties, for a fixed truth degree $d\in[0,1[$:
\begin{enumerate}[leftmargin=*]
\item $I \models \Pi_\phi^\odot$ implies $I(\mathit{sat}) = 1$;
\item $I \models \Pi_\phi^\odot$ implies either $I(x_i^T) = 1 \wedge I(x_i^F) = d$, or $I(x_i^F) = 1 \wedge  I(x_i^T) = d$, for all $i \in [1..n]$;
\item $I \models \Pi_\phi^\odot$ and $I(\mathit{sat}) = 1$ implies $I(x_i^T) = I(x_i^F) = 1$, for all $i \in [m+1..n]$;
\item $J \subset I$ and $J \models (\Pi_\phi^\odot)^I$ implies $J(\mathit{sat}) = d$ and either $I(x_i^T) = 1 \wedge I(x_i^F) = d$, or $I(x_i^F) = 1 \wedge I(x_i^T) = d$, for all $i \in [1..n]$.
\end{enumerate}

We will then define a mapping between assignments for $x_1,\ldots,x_m$ and interpretations of $\Pi_\phi^\odot$.
Let $\nu$ be a Boolean assignment for $x_1,\ldots,x_m$.
Define $I_\nu^d$ to be the interpretation such that:
$I_\nu^d(x_i^T)$ equals 1 if $\nu(x_i) = 1$, and $d$ otherwise, for all $i \in [1..m]$;
$I_\nu^d(x_i^F)$ equals 1 if $\nu(x_i) = 0$, and $d$ otherwise, for all $i \in [1..m]$;
$I_\nu(x_i) = 1$ for all $i \in [m+1..n]$; and
$I_\nu(\mathit{sat}) = 1$.
Moreover, for an extended Boolean assignment for $x_1,\ldots,x_n$, we define $I_{\nu'}$ to be the interpretation such that:
$I_{\nu'}^d(x_i^T)$ equals 1 if $\nu'(x_i) = 1$, and $d$ otherwise, for all $i \in [1..n]$;
$I_{\nu'}^d(x_i^F)$ equals 1 if $\nu'(x_i) = 0$, and $d$ otherwise, for all $i \in [1..n]$; and
$I_{\nu'}(\mathit{sat}) = d$.
These mappings will allow us to define one-to-one mappings between satisfying assignments of $\phi$ and stable models of $\Pi_\phi^\odot$, and between unsatisfying assignments of $\phi$ and minimal models of reducts (counter models of $\Pi_\phi^\odot$).

\paragraph{Proof of (ii).} 
We adapt the construction by \cite{DBLP:journals/amai/EiterG95}.
The program $\Phi_\phi$ is the following:
\begin{eqnarray}
    x_i^T \veebar x_i^F \leftarrow 1 && \forall i \in [1..n] \label{eq:hard:1:apx} \\
    x_i^T \leftarrow \mathit{sat} \quad x_i^F \leftarrow \mathit{sat} \quad \mathit{0} \leftarrow \naf \mathit{sat} && \forall i \in [m+1..n] \label{eq:hard:3:apx} \\
    \mathit{sat} \leftarrow \sigma(L_{k,1}) \barwedge \sigma(L_{k,2}) \barwedge \sigma(L_{k,3}) && \forall i \in [1..k] \label{eq:hard:4:apx}
\end{eqnarray}
The program $\Pi_\phi^\veebar$ has the four properties given above for $d = 0$.
Any model of $\Pi_\phi^\veebar$ is of the form $I_\nu^0$, for some assignment $\nu$ for $x_1,\ldots,x_m$.
If we consider the reduct $(\Pi_\phi^\veebar)^{I_\nu^0}$, the rule $0 \leftarrow \naf\mathit{sat}$ is replaced by $0 \leftarrow 0$.
Any minimal model strictly contained in $I_\nu$ will be of the form $J_{\nu'}$ for some assignment $\nu'$ extending $\nu$.
Such a $J_{\nu'}$ would imply that $\nu'(\psi) = 0$, and therefore $\nu(\phi) = 0$.
On the other hand, if such a $J_{\nu'}$ does not exist, it means that $\mathit{sat}$ is necessarily 1;
iff there is $i \in [1..k]$ such that $\sigma(L_{k,1}) \barwedge \sigma(L_{k,2}) \barwedge \sigma(L_{k,3})$ is necessarily 1;
iff all $\nu'$ extending $\nu$ are such that $\nu'(\psi) = 1$;
iff $\nu(\psi) = 1$.
Hence, we have that $\phi$ is satisfiable iff $\Pi_\phi^\veebar$ is coherent.

To complete this part of the proof, it is enough to replace (\ref{eq:hard:4:apx}) by
\begin{eqnarray}
    \mathit{sat} \leftarrow \sigma(L_{k,1}) \otimes \sigma(L_{k,2}) \otimes \sigma(L_{k,3}) && \forall i \in [1..k] \label{eq:hard:alt}
\end{eqnarray}
because any model and counter model of $\Pi_\phi^\veebar$ give a Boolean interpretation to $\sigma(L_{k,1}) \otimes \sigma(L_{k,2}) \otimes \sigma(L_{k,3})$.

\paragraph{Proof of (iii).}
This is essentially folklore. Having $\oplus$ in rule bodies allows to crispify a variable $p$ by means of the common pattern $p \leftarrow p \oplus p$.
The program $\Pi_\phi^\oplus$ is thus
\begin{eqnarray}
    x_i^T \oplus x_i^F \leftarrow 1 \quad x_i^T \leftarrow x_i^T \oplus x_i^T \quad x_i^F \leftarrow x_i^F \oplus x_i^F && \forall i \in [1..n]  \\
    x_i^T \leftarrow \mathit{sat} \quad x_i^F \leftarrow \mathit{sat} \quad 0 \leftarrow \naf \mathit{sat} \quad \mathit{sat} \leftarrow \mathit{sat} \oplus \mathit{sat} && \forall i \in [m+1..n]  \\
    \mathit{sat} \leftarrow \sigma(L_{k,1}) \barwedge \sigma(L_{k,2}) \barwedge \sigma(L_{k,3}) && \forall i \in [1..k] \label{eq:hard:iii}
\end{eqnarray}
The same argument used for (ii) proves that $\phi$ is satisfiable iff $\Pi_\phi^\oplus$ is coherent.
The same holds if (\ref{eq:hard:iii}) is replaced by (\ref{eq:hard:alt}).

\paragraph{Proof of (i).}
This is the most sophisticated construction.
The program $\Pi_\phi^\otimes$ is 
\begin{eqnarray}
    x_i^T \oplus x_i^F \leftarrow 0.5 && \forall i \in [1..n]  \\
    x_i^T \otimes x_i^T \otimes x_i^T \leftarrow x_i^T \otimes x_i^T && \forall i \in [1..n] \\
    x_i^F \otimes x_i^F \otimes x_i^F \leftarrow x_i^F \otimes x_i^F && \forall i \in [1..n] \\
    x_i^T \leftarrow \mathit{sat} \quad x_i^F \leftarrow \mathit{sat} \quad 0 \leftarrow \naf \mathit{sat} \quad \mathit{sat} \leftarrow 0.5 && \forall i \in [m+1..n]  \\
    \mathit{sat} \leftarrow \sigma(L_{k,1}) \otimes \sigma(L_{k,2}) \otimes \sigma(L_{k,3}) && \forall i \in [1..k] 
\end{eqnarray}
This program $\Pi_\phi^\otimes$ has the four properties given at the beginning of this proof, but for $d = 0.5$.
(Note that rule $\mathit{sat} \leftarrow 0.5$ was added to have a uniform proof with the previous parts, but the construction would work also without such a rule.)
In fact, all atoms must be assigned a truth degree of 0.5 or 1.
Hence, the interpretation of $\sigma(L_{k,1}) \otimes \sigma(L_{k,2}) \otimes \sigma(L_{k,3})$ will be 1 if $\sigma(L_{k,1}),\sigma(L_{k,2}),\sigma(L_{k,3})$ are 1, and less than or equal to 0.5 otherwise.
We can thus rely on the argument given in the proof of~(ii).
\end{proof}

\ThmShift*
\begin{proof}
Since the shift is performed independently on each rule of $\Pi$, it suffices to show $\Pi'' \cup \{p_1 \odot \cdots \odot p_n \leftarrow \beta\} \equiv_{\At(\Pi)} \Pi'' \cup \shift(\{p_1 \odot \cdots \odot p_n \leftarrow \beta\})$, where $\Pi'' \cup \{p_1 \odot \cdots \odot p_n \leftarrow \beta\}=\Pi$, $n \geq 2$, and $\odot \in \{\oplus,\otimes,\veebar\}$.
To simplify the presentation, $\beta$ is assumed to be a propositional atom.
Moreover, since $\Pi$ is HCF, w.l.o.g. we can assume that, for $1 \leq i < j \leq n$, $p_i$ does not reach $p_j$ in $\G_\Pi$.
In each part of the proof, we will provide a one-to-one mapping between the (minimal) models of the original program and the models of shifted program.
Moreover, we will give a mapping of the counter model of the original program into the counter models of the shifted program, and \emph{vice versa}.

\paragraph{Proof for $\oplus$.}
$I \models \Pi'' \cup \{p_1 \oplus \cdots \oplus p_n \leftarrow \beta\}$ iff $I \models \Pi'' \cup \shift(\{p_1 \oplus \cdots \oplus p_n \leftarrow \beta\})$ holds because
$I(p_1) + \cdots + I(p_n) \geq I(\beta)$ iff
$$I(p_i) \geq I(\beta) + \sum_{j \in [1..n], j \neq i} (1 - I(p_j)) - (n-1) = I(\beta) - \sum_{j \in [1..n], j \neq i} I(p_j)$$ 
for all $i \in [1..n]$.
Let $I$ be a model of the two programs.

For all $J \subset I$, it holds that $J \models (\Pi'')^I \cup \{p_1 \oplus \cdots \oplus p_n \leftarrow \beta\}^I$ implies that $J \models (\Pi'')^I \cup \shift(\{p_1 \oplus \cdots \oplus p_n \leftarrow \beta\})^I$ because 
$J(p_1) + \cdots + J(p_n) \geq J(\beta)$ iff
$$J(p_i) \geq J(\beta) + \sum_{j \in [1..n], j \neq i} (1 - J(p_j)) - (n-1) = J(\beta) - \sum_{j \in [1..n], j \neq i} J(p_j),$$ 
for all $i \in [1..n]$, which implies
$$J(p_i) \geq J(\beta) + \sum_{j \in [1..n], j \neq i} (1 - I(p_j)) - (n-1) = J(\beta) - \sum_{j \in [1..n], j \neq i} I(p_j)$$ 
because by assumption $J(p_j) \leq I(p_j)$ for all $p_j \in [1..n]$.

For the converse direction, we show that for any interpretation $J \subset I$ such that $J \models (\Pi'')^I \cup \shift(\{p_1 \oplus \cdots \oplus p_n \leftarrow \beta\})^I$, there is $K$ such that $J \subseteq K \subset I$ and $K \models (\Pi'')^I \cup \{p_1 \oplus \cdots \oplus p_n \leftarrow \beta\}^I$.
Let us assume that $\{p_1 \oplus \cdots \oplus p_n \leftarrow \beta\} \neq \emptyset$, and that $J(p_i) < I(p_i)$ for some $i \in [1..n]$, otherwise the proof is immediate.
We define the following non-deterministic sequence:
$K_0 := J$;
for $i \in [0..n-1]$, 
$K_{i+1}$ is any subset minimal model of $(\Pi'')^I$ such that 
$K_i \subseteq K_{i+1} \subset I$, and $K_{i+1} = \min(I(p_{n-i}), m)$, where $m = \max(K_i(p_{n-i}), K_i(\beta) - \sum_{j \in [1..n], j \neq i} K_i(p_j))$.
The sequence is well defined because in $K_{i+1}$ we are possibly increasing the truth degree of $p_{n-i}$, which cannot cause an increase of any $p_j$ with $j < n-i$ by assumption.
Intuitively, we possibly increase the truth degree of $p_1,\ldots,p_n$ in order to satisfy the original rule $p_1 \oplus \cdots \oplus p_n \leftarrow \beta$, and we do this by preferring atoms with higher indices.
Hence, we have $K_n \subset I$ and $K_n \models (\Pi'')^I \cup \{p_1 \oplus \cdots \oplus p_n \leftarrow \beta\}^I$.

\paragraph{Proof for $\otimes$.}
For an interpretation $I$, define $I'$ to be such that:
$I'(p) = I(p)$ for all $p \in \At(\Pi)$;
$I'(q)$ equals 1 if $I(\beta) > 0$, and 0 otherwise.
We follow the line of the previous proof.
Let $I$ be an interpretation such that $I(\beta) > 0$, otherwise the proof is immediate.
Then, $I(q) = 1$, and 
$I$ is a minimal model of $\Pi'' \cup \{p_1 \otimes \cdots \otimes p_n \leftarrow \beta\}$ if and only if
 $I'$ is a minimal model of $\Pi'' \cup \shift(\{p_1 \otimes \cdots \otimes p_n \leftarrow \beta\})$ holds because
$I(p_1) + \cdots + I(p_n) - (n-1) \geq I(\beta)$ iff
$I(p_i) \geq I(\beta) + \sum_{j \in [1..n], j \neq i} (1 - I(p_j))$,
for all $i \in [1..n]$.
Let $I$ be a minimal model of $\Pi$ with $I(\beta) > 0$.

For all $J \subset I$, we have that $J \models (\Pi'')^I \cup \{p_1 \otimes \cdots \otimes p_n \leftarrow \beta\}^I$ implies that $J' \models (\Pi'')^{I'} \cup \shift(\{p_1 \otimes \cdots \otimes p_n \leftarrow \beta\})^{I'}$ because 
$J(p_1) + \cdots + J(p_n) - (n-1) \geq J(\beta)$ iff
$J(p_i) \geq J(\beta) + \sum_{j \in [1..n], j \neq i} (1 - J(p_j))$, 
for all $i \in [1..n]$, which itself implies
$J(p_i) \geq J(\beta) + \sum_{j \in [1..n], j \neq i} (1 - I(p_j))$
since by assumption $J'(p_j) = J(p_j) \leq I(p_j)$ for all $p_j \in [1..n]$.

For the converse direction, we only change the non-deterministic sequence from the previous proof as follows:
$K_0 := J$;
for $i \in [0..n-1]$, 
$K_{i+1}$ is any subset minimal model of $(\Pi'')^I$ such that 
$K_i \subseteq K_{i+1} \subset I$, and $K_{i+1} = \min(I(p_{n-i}), m)$, where $m = \max(K_i(p_{n-i}), K_i(\beta) + \sum_{j \in [1..n], j \neq i} (1-K_i(p_j)))$.
We have $K_n \subset I'$.

\paragraph{Proof for $\veebar$.}
Given an interpretation $I$, define $I'$ to be such that:
$I'(p) = I(p)$ for every $p \in \At(\Pi)$;
$I'(q_n) = 1$; and
for $i \in [1..n-1]$, $I'(q_i)$ is equal to 1 if $I(p_i) > \max\{I(p_j) \mid j \in [i+1..n]\}$, and 0 otherwise.
Following the line of the previous two proofs,
$I$ is a minimal model of $\Pi'' \cup \{p_1 \veebar \cdots \veebar p_n \leftarrow \beta\}$ if and only if $I'$ is a minimal model of $\Pi'' \cup \shift(\{p_1 \veebar \cdots \veebar p_n \leftarrow \beta\})$. This holds because
$\max\{I(p_1), \ldots, I(p_n)\} \geq I(\beta)$ iff
$I(p_i) \geq I(\beta)$ for the index $i \in [1..n]$ such that
$I(p_i) \geq \max\{I(p_j) \mid j \in [1..i-1]\}$, and
either $I(p_i) > \max\{I(p_j) \mid j \in [i+1..n]\}$ or $i = n$.
Let $I$ be a minimal model of the program $\Pi$.

For $J \subset I$, define $J''$ to be such that:
$J''(p) = J(p)$ for every $p \in \At(\Pi)$; and 
$J''(q_i) = J'(q_i)$ for all $i \in [1..n]$.
Then $J \models (\Pi'')^I \cup \{p_1 \veebar \cdots \veebar p_n \leftarrow \beta\}^I$ implies $J'' \models (\Pi'')^{I'} \cup \shift(\{p_1 \veebar \cdots \veebar p_n \leftarrow \beta\})^{I'}$ since 
$\max\{J(p_1), \ldots, J(p_n)\} \geq J(\beta)$ iff
$J(p_i) \geq J(\beta)$ for the index $i \in [1..n]$ with
$J(p_i) \geq \max\{I(p_j) \mid j \in [1..i-1]\}$, and
either $i = n$ or $J(p_i) > \max\{J(p_j) \mid j \in [i+1..n]\}$. This holds
because by assumption $J''(p_j) = J(p_j) \leq I(p_j)$ for all $p_j \in [1..n]$.

As for the other direction, again, we only change the non-deterministic sequence as follows:
$K_0$ is such that $K_0(p) = J(p)$ for all $p \in \At(\Pi)$, and $K_0(q_i) = J'(q_i)$ for all $i \in [1..n]$;
for $i \in [0..n-1]$, 
$K_{i+1}$ is any subset minimal model of $(\Pi'')^I$ such that 
$K_i \subseteq K_{i+1} \subset I$, and $K_{i+1} = \min(I(p_{n-i}), m)$, where $m$ equals $K_i(\beta)$ if $\max_{j \in [1..n], j \neq n-i}{K_i(p_j)} < K_i(\beta)$, and $K_i(p_{n-1})$ otherwise.
We have $K_n \subset I'$.
\end{proof}

\ThmSmt*
\begin{proof}
We use structural induction to prove that $I(\alpha) = f(\alpha)^{\A_I}$ holds for any expression or term $\alpha$, and for $f \in \{\out,\inn\}$.
\begin{itemize}[leftmargin=*]
\item
The base cases are immediate:
for $c \in [0,1]$, $I(c) = c^{\A_I} = f(c)^{\A_I}$ by definition;
for $p \in \At(\Pi)$, $I(p) = p^{\A_I} = f(p)^{\A_I}$ by definition.

\item
For $\naf$, assuming that the claim holds for $\alpha$, we have
$$I(\naf\alpha) = 1 - I(\alpha) = 1 - \out(\alpha)^{\A_I} = f(\naf\alpha)^{\A_I}.$$

\item
For $\oplus$, assuming that the claim holds for $\alpha$ and $\beta$, we have
\begin{align*}
I(\alpha \oplus \beta) & {}= \min(I(\alpha) + I(\beta), 1) = \min(f(\alpha)^{\A_I} + f(\beta)^{\A_I}, 1) \\
 & {}= \ite(f(\alpha) + f(\beta) \leq 1, f(\alpha) + f(\beta), 1)^{\A_I} = f(\alpha \oplus \beta)^{\A_I}.
\end{align*}

\item
For $\otimes$, assuming that the claim holds for $\alpha$ and $\beta$, we have
\begin{align*}
I(\alpha \otimes \beta) & {}= \max(I(\alpha) + I(\beta) - 1, 0) = \max(f(\alpha)^{\A_I} + f(\beta)^{\A_I} - 1, 0) \\
 & {}= \ite(f(\alpha) + f(\beta) - 1 \geq 0, f(\alpha) + f(\beta) - 1, 0)^{\A_I} = f(\alpha \otimes \beta)^{\A_I}.
\end{align*}

\item
For $\veebar$, assuming that the claim holds for $\alpha$ and $\beta$, we have
\begin{align*}
I(\alpha \veebar \beta) & {}= \max(I(\alpha), I(\beta)) = \max(f(\alpha)^{\A_I}, f(\beta)^{\A_I}) \\
 & {}= \ite(f(\alpha) \geq f(\beta), f(\alpha), f(\beta))^{\A_I} = f(\alpha \veebar \beta)^{\A_I}.
\end{align*}

\item
For $\barwedge$, assuming that the claim holds for $\alpha$ and $\beta$, we have
\begin{align*}
I(\alpha \barwedge \beta) & {}= \min(I(\alpha), I(\beta)) = \min(f(\alpha)^{\A_I}, f(\beta)^{\A_I}) \\
 & {}= \ite(f(\alpha) \leq f(\beta), f(\alpha), f(\beta))^{\A_I} = f(\alpha \barwedge \beta)^{\A_I}.
\end{align*}
\end{itemize}
We can thus conclude that $I \models \Pi$ if and only if $\A_I$ is a $\Sigma$-model of the theory
$\{p \in [0,1] \mid p \in \At(\Pi)\} \cup \{\out(r) \mid r \in \Pi\}$.
Moreover, if $I \in \SM(\Pi)$ then there is no $J \subset I$ such that $J \models \Pi^I$, which is the case if and only if $\A_I$ also satisfies formula $\phi_\inn$.
\end{proof}

\ThmComp*
\begin{proof}
Let $\Pi'$ be $\shift(\simp(\Pi))$, and $\Pi'' = \bool^-(\Pi')$.
By Proposition~\ref{prop:simp} and Theorem~\ref{thm:shift}, we know that  $\Pi \equiv_{\At(\Pi)} \Pi'$.
Moreover, if $\Pi \setminus \bool(\Pi)$ is acyclic then $\Pi''$ is acyclic.
From the correctness of the completion proved by \citeN{DBLP:journals/tplp/JanssenVSC12}, and since
$\supp(p,\heads(\Pi''))^{\A_I} = \max\{\beta^{\A_I} \mid p \leftarrow \beta \in \Pi''\} = \max\{I(\beta) \mid p \leftarrow \beta \in \Pi''\}$ captures the notion of support of $p$,
we have that
$I \in \SM(\Pi'')$ iff $\A_I \models \comp(\Pi'')$.
Hence, the models of $\rcomp(\Pi)$ are the structures $\A_I^r$ such that $I \in \SM(\Pi'')$ satisfying the following condition:
$I(b_p)$ equals 1 if $I(p) > 0$, and 0 otherwise.
These are exactly the stable models of $\Pi'$, which concludes the proof.
\end{proof}

\LemRank*
\begin{proof}
We first prove the claim for programs without $\oplus$.
Let $J_0$ be the interpretation mapping everything to 0, and $J_{i+1}:=\T_{\Pi^I}(J_i)$, for all $i\geq 0$.
For every $i\geq 0$ and $p \in \At(\Pi)$, if $J_i(p) < J_{i+1}(p)$, then there is a rule $p \leftarrow \beta \in \Pi^I$ with $J_{i+1}(p) = J_i(\beta)$.
In this case, for each atom $q$ (including numeric constants) occurring $\beta$, we say that $p$ is \emph{inferred} by $q$.
In particular, since $\beta$ can only contain $\barwedge$ and $\veebar$, we have the following property:
($*$) $J_{i+1}(p) \leq J_i(q)$.
Let $n = |\At(\Pi)|$ be the number of atoms in $\Pi$.
We prove that any chain of inferred atoms has length at most $n+1$, which implies that $n$ applications of $\T_{\Pi^I}$ give the fixpoint of the operator.
Suppose on the contrary that there are $p_0, \ldots, p_{n+1}$ such that
$p_0$ is a numeric constant and 
$p_{i+1} \in \At(\Pi)$ is inferred by $p_i \in \At(\Pi)$, for all $i \in [0..n]$.
Since $n = |\At(\Pi)|$, there exist $1\le j < k \le n+1$ such that $p_j=p_k$.
Hence, from $J_i(p) < J_{i+1}(p)$ we have $J_{i+1}(p_{i+1}) > J_{i}(p_{i+1})$ for $i \in [0..n]$, and thus $J_k(p_k) > J_{k-1}(p_k) \geq J_j(p_k)$ (where the last inequality is due to the monotonicity of $\T_{\Pi^I}$).
From ($*$) we have $J_{i+1}(p_{i+1}) \leq J_i(p_i)$ for $i \in [0..n]$, and thus $J_k(p_k) \leq J_j(p_j) = J_j(p_{k})$.
Therefore, we have $J_k(p_k) > J_j(p_k)$ and $J_k(p_k) \leq J_j(p_{k})$, that is, a contradiction.

Let us now add non-recursive $\oplus$ in rule bodies.
If there is $i \in [0..n]$ such that $p_{i+1}$ and $p_i$ do not satisfy ($*$), i.e., $J_{i+1}(p_{+1}) > J_i(p_i)$, then $\beta$ must contain some occurrence of $\oplus$.
Since $\oplus$ is non-recursive by assumption, $\{p_j \mid i \in [1..i]\}$ and $\{p_j \mid [i+1..n+1]\}$ are disjoint sets.
Either $p_1,\ldots,p_i$ or $p_{i+1},\ldots,p_{n+1}$ must have a repeated atom, and argument used before gives a contradiction.
\end{proof}

\ThmOcomp*
\begin{proof}
Let $\Pi'$ be $\shift(\simp(\Pi))$.
From Proposition~\ref{prop:simp} and Theorem~\ref{thm:shift} we have $\Pi \equiv_{\At(\Pi)} \Pi'$.
Moreover, $\Pi'$ has atomic heads and non-recursive $\oplus$ in rule bodies.
We show that stable models of $\Pi'$ and $\Sigma$-models of $\ocomp(\Pi')$ are related.

First, notice that for any structure \A and set of atoms $A$, $\rank(A)^\A$ equals $\max\{r_p^\A \mid p \in A\}$ if $A \neq \emptyset$, and 0 otherwise.
Moreover, $\osupp(p, \heads(p,\Pi'))^\A=1$ if there is $p \leftarrow \beta \in \heads(p,\Pi)$ such that $p^\A = \beta^\A$ and $r_p^\A = 1 + \rank(\pos(\beta))^\A$.

\paragraph{$(\Rightarrow)$}
Let $I \in \SM(\Pi')$.
Let $J_0$ be the interpretation mapping everything to 0, and $J_{i+1}$ be $\T_{\Pi'^I}(J_i)$, for $i\geq 0$.
By Lemma~\ref{lem:rank}, $J_{n+1} = J_n$.
Let $r$ be the ranking associated with $I$, i.e., $r(p)$ equals the minimum index $i \in [1..n]$ such that $J_i(p) = J_n(p)$.

We now use induction on the rank of inferred atoms to prove the following: $\A_I^o \models p = \out(\beta) \wedge r_p = 1 + \rank(\pos(\beta)))$.
For all $p \in \At(\Pi)$ such that $J_n(p) > 0$ and $r(p) = 1$, there is a rule $p \leftarrow \beta \in \Pi'^I$ such that $J_n(\beta) = J_n(p)$ and $\beta$ only contains numeric constants;
in this case $\A_I^o \models p = \out(\beta) \wedge r_p = 1 + \rank(\pos(\beta)))$.
For $m \in [1..n-1]$, and for all $p \in \At(\Pi)$ such that $J_n(p) > 0$ and $r(p) = m+1$, there is a rule $p \leftarrow \beta \in \Pi'^I$ such that $J_n(\beta) = J_n(p)$ and $q \in \pos(\beta)$ implies $r(q) \leq m$;
since the claim is true for all $q \in \pos(\beta)$, and at least one of them must satisfy $r(q) = m$, we have $\A_I^o \models p = \out(\beta) \wedge r_p = 1 + \rank(\pos(\beta)))$.

That $\A_I^o \models \comp(\Pi')$ follows by the fact that the completion captures the notion of supported model.
Hence, $\A_I^o \models \ocomp(\Pi')$.

\paragraph{$(\Leftarrow)$}
Let $\A$ be a $\Sigma$-model of $\ocomp(\Pi')$, and let $I := I_\A$.
We shall show that $I_\A \in \SM(\Pi')$.
Let $J_0$ be the interpretation mapping everything to 0, and $J_{i+1}$ be $\T_{\Pi'^I}(J_i)$, for $i\geq 0$.

We use induction on $r_p^\A$ to show that $J_{r_p^\A}(p) = I(p)$.
If $p^\A > 0$ and $r_p^\A = 1$, then there is $p \leftarrow \beta \in \Pi'$ such that $p^\A = \beta^\A$ and $\pos(\beta) = \emptyset$;
in this case $J_1(p) = I(p)$.
If $p^\A > 0$ and $r_p^\A = m+1$ for some $m \in [1..n-1]$, there is $p \leftarrow \beta \in \Pi'$ such that $p^\A = \beta^\A$ and $\max\{r_q^\A \mid q \in \pos(\beta)\} = m$;
since $J_m(q) = I(q)$ for all $q \in \pos(\beta)$ by the induction hypothesis, we have $J_{m+1}(p) = I(\beta) = I(p)$.

The proof is thus complete.
\end{proof}

\end{document}